\documentclass{article}

\usepackage[utf8]{inputenc} 
\usepackage[T1]{fontenc}    
\usepackage{hyperref}       
\usepackage{url}            
\usepackage{booktabs}       
\usepackage{amsfonts}       
\usepackage{nicefrac}       
\usepackage{microtype}      
\usepackage[dvipsnames]{xcolor}         
\usepackage{pifont}
\usepackage{subcaption}
\usepackage{fullpage}
\usepackage{natbib}
\usepackage{subcaption}
\usepackage{enumitem}
\usepackage{amsmath}
\usepackage{amsthm}
\usepackage{wrapfig}
\usepackage[capitalize,noabbrev]{cleveref}
\usepackage{mleftright}

\usepackage{graphicx}       

\setlist[itemize]{leftmargin=2.5em}
\setlist[itemize]{leftmargin=2.5em}
\setlist[enumerate]{leftmargin=2.5em}

\theoremstyle{plain}
\newtheorem{theorem}{Theorem}[section]
\newtheorem{proposition}[theorem]{Proposition}
\newtheorem{lemma}[theorem]{Lemma}
\newtheorem{corollary}[theorem]{Corollary}
\theoremstyle{definition}

\theoremstyle{remark}
\newtheorem{remark}[theorem]{Remark}
\theoremstyle{remark}
\newtheorem{fact}[theorem]{Fact}

\newcommand{\Pm}{\mathcal{P}}
\usepackage[textsize=tiny]{todonotes}

\usepackage{bm}

\usepackage{dsfont}
\usepackage{nicematrix}

\title{The Evolution of Statistical Induction Heads: \\ In-Context Learning Markov Chains}
\date{}
\author{
\begin{tabular}{c}
Benjamin L. Edelman$^1$,
Ezra Edelman$^{2}$,
Surbhi Goel$^{2}$,
Eran Malach$^1$,
Nikolaos Tsilivis$^{3}$\thanks{Work done while visiting Harvard University.}\\
\\
\normalsize{$^1$Harvard University, $^2$University of Pennsylvania, $^3$NYU}\\
\\
\normalsize{\texttt{bedelman@g.harvard.edu, ezrae@cis.upenn.edu}} \\
\normalsize{\texttt{surbhig@cis.upenn.edu, emalach@g.harvard.edu, nt2231@nyu.edu}}
\end{tabular}
}

\begin{document}

\maketitle

\def\doublecolumn{0}

\begin{abstract}
Large language models have the ability to generate text that mimics patterns in their inputs. We introduce a simple Markov Chain sequence modeling task in order to study how this in-context learning (ICL) capability emerges. In our setting, each example is sampled from a Markov chain drawn from a prior distribution over Markov chains. Transformers trained on this task form \emph{statistical induction heads} which compute accurate next-token probabilities given the bigram statistics of the context. During the course of training, models pass through multiple phases: after an initial stage in which predictions are uniform, they learn to sub-optimally predict using in-context single-token statistics (unigrams); then, there is a rapid phase transition to the correct in-context bigram solution. We conduct an empirical and theoretical investigation of this multi-phase process, showing how successful learning results from the interaction between the transformer's layers, and uncovering evidence that the presence of the simpler unigram solution may delay formation of the final bigram solution. We examine how learning is affected by varying the prior distribution over Markov chains, and consider the generalization of our in-context learning of Markov chains (ICL-MC) task to $n$-grams for $n > 2$.
\end{abstract}

\section{Introduction}

Large language models (LLMs) exhibit a remarkable ability to perform \emph{in-context learning} (ICL): learning from patterns in their input context \citep{Bro+20,dong2022survey}. 
The ability of LLMs to adaptively learn from context is profoundly useful, yet the underlying mechanisms of this emergent capability are not fully understood.

    


In an effort to better understand ICL, some recent works propose to study ICL in controlled synthetic settings---in particular, training transformers on mathematically defined tasks which require learning from the input context. For example, a recent line of works studies the ability of transformers to perform ICL of standard supervised learning problems such as linear regression \citep{Gar+22, akyurek2022learning, li2023transformers, wu2023many}. Studying these well-understood synthetic learning tasks enables fine-grained control over the data distribution, allows for comparisons with established supervised learning algorithms, and facilitates the examination of the in-context ``algorithm'' implemented by the network. That said, these supervised settings are reflective specifically of \emph{few-shot learning}, which is only a special case of the more general phenomenon of networks incorporating patterns from their context into their predictions. A few recent works \citep{bietti2023birth, Xie+22} go beyond the case of cleanly separated in-context inputs and outputs, studying in-context learning on distributions based on discrete stochastic processes.

The goal of this work is to propose and analyze a simple synthetic setting for studying ICL. To achieve this, we consider $n$-gram models \citep{brown1992class, shannon1948mathematical, chomsky1956three}, one of the simplest and oldest methods for language modeling. An $n$-gram language model predicts the probability of a token based on the preceding $n-1$ tokens, using fixed-size chunks ($n$-grams) of text data to capture linguistic patterns. Our work studies ICL of $n$-gram models, where the network needs to compute the conditional probability of the next token based on the statistics of the tokens observed in the input context, rather than on the statistics of the entire training data. We mainly focus on the simple case of $n=2$; i.e., bigram models, which can be represented as Markov chains.
We therefore consider ICL of Markov chains (ICL-MC): we train a transformer on sequences of tokens, where each sequence is produced by a different Markov chain, generated using a different transition matrix (see Figure~\ref{fig:main-plot} (left)).

\begin{figure}[t!]
\centering
\begin{subfigure}{0.45\textwidth}
    \includegraphics[trim= 0 100 250 100, clip, width=0.95\linewidth]{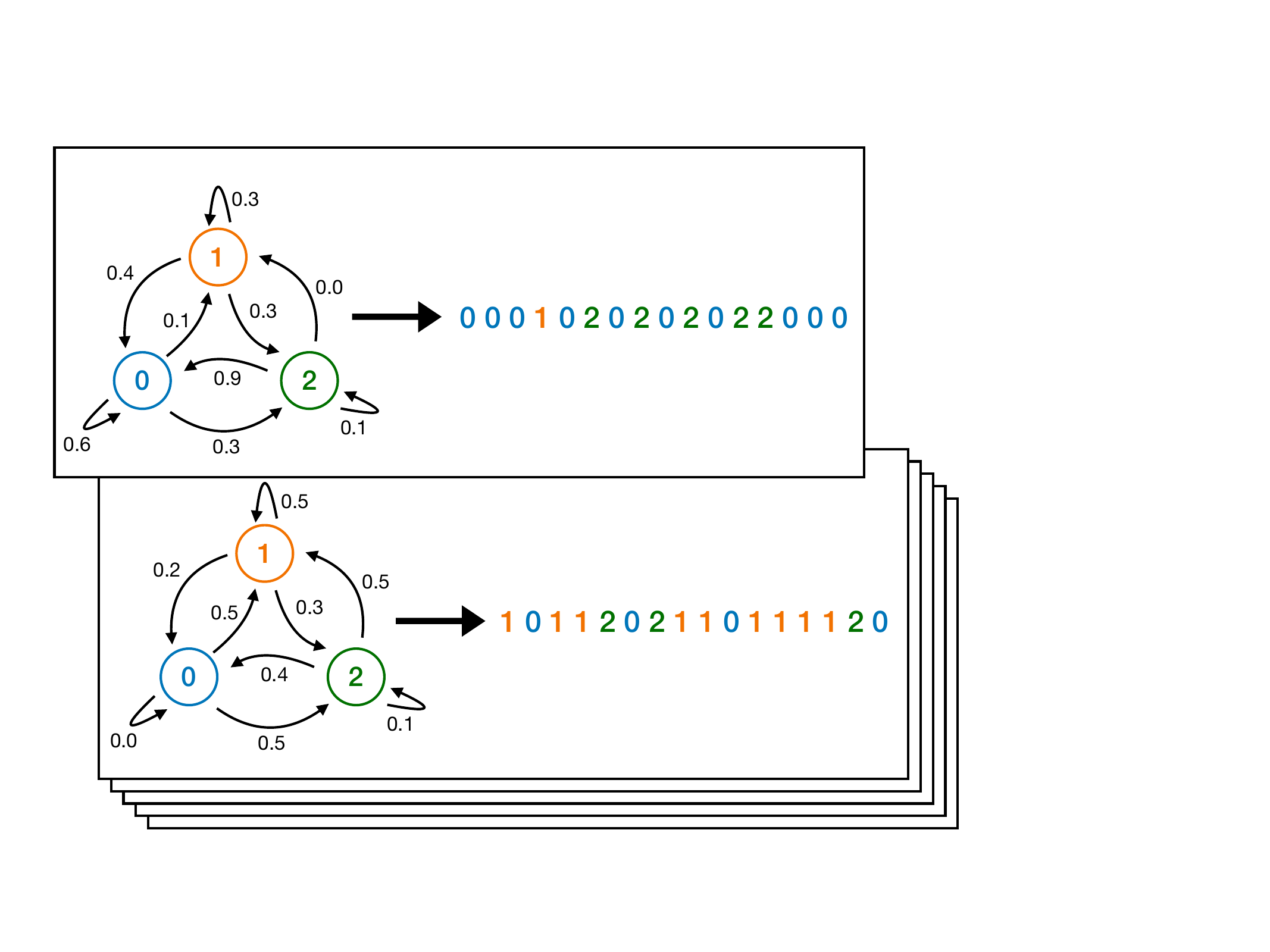}
\end{subfigure}
  \begin{subfigure}{0.54\textwidth}
    \includegraphics[trim= 0 0 0 0, clip, width=0.95\linewidth]{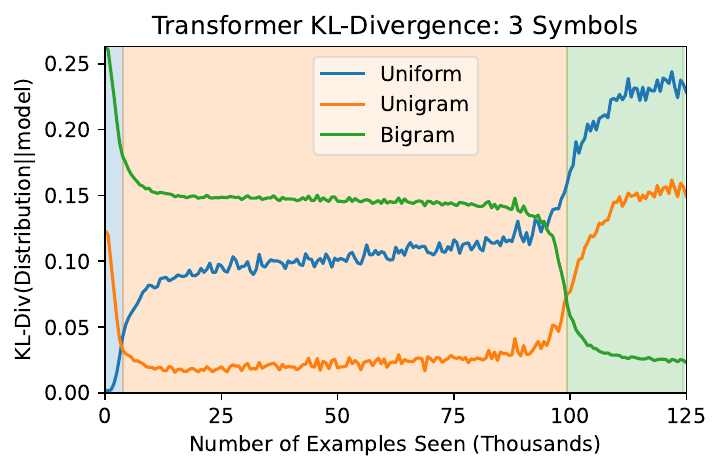}
\end{subfigure}    
\caption{(\textbf{left}) We train small transformers to perform in-context learning of Markov chains (ICL-MC)---next-token prediction on the outputs of Markov chains. Each training sequence is generated by sampling a transition matrix from a prior distribution, and then sampling a sequence from this Markov chain. (\textbf{right}) Distance of a transformer's output distribution to several well-defined strategies over the course of training on our in-context Markov chain task. The model passes through three stages: (1) predicting a uniform distribution, (2) predicting based on in-context unigram statistics, (3) predicting based on in-context bigram statistics. Shading is based on the minimum of the curves.} \label{fig:main-plot}
\end{figure}

By studying ICL-MC, we are able to replicate and study multiple phenomena that have been observed in ICL for LLMs, and identify new ones. We demonstrate our findings using a combination of empirical observations on transformers trained from scratch on ICL-MC and theoretical analysis of a simplified linear transformer. Our key findings are summarized below:
\begin{itemize}[leftmargin=*]
    \item \textbf{Transformers learn statistical induction heads to optimally solve ICL-MC}.
    Prior work studying ICL in transformers revealed the formation of \emph{induction heads} \citep{elhage2021mathematical}, a circuit that looks for recent occurrence(s) of the current token, and boosts the probabilities of tokens which followed in the input context.
    We show that in order to solve ICL-MC, transformers learn \textit{statistical} induction heads that are able to compute the correct \emph{conditional (posterior) probability} of the next token given all previous occurrences of the prior token (see the attention patterns in Figure~\ref{fig:attn}). We show that these statistical induction heads lead to the transformers achieving performance approaching that of the Bayes-optimal predictor. \looseness=-1
    \item \textbf{Transformers learn predictors of increasing complexity and undergo a phase transition when increasing complexity.} We observe that transformers display \textit{phase transitions} when learning Markov chains---learning appears to be separated into phases, with fast drops in loss between the phases. We are able to show that different phases correspond to learning models of increased complexity---unigrams, then bigrams (see Figure~\ref{fig:main-plot})---and characterize the transition between the phases. We also consider the $n$-gram generalization of our setting where the next token is generated based on the previous $n-1$ tokens.
    \item \textbf{Simplicity bias may slow down learning.} We provide evidence that the model's inherent bias towards simpler solutions (in particular, in-context unigrams) causes learning of the optimal solution to be delayed. Changing the distribution of the in-context examples to remove the usefulness of in-context unigrams leads to faster convergence, even when evaluated on the original distribution.
    \item \textbf{Alignment of layers is crucial.} We show that the transition from a phase of learning the simple-but-inadequate solution to the complex-and-correct solution happens due to an alignment between the layers of the model: the learning signal for the first layer is tied to the extent to which the second layer approaches its correct weights.
    \item \textbf{Alternating patterns in positional embeddings.} When we train transformers with relative position embeddings, the theoretical optimization analysis of our simplified model indicates that along the way to the correct solution, the first layer develops a bias toward looking back an \emph{odd} number of tokens, even though only looking back by 1 is clearly useful. We empirically observe this curious phenomenon in real transformers as well.
\end{itemize}

\begin{figure*}[t]
    \includegraphics[width=\textwidth]{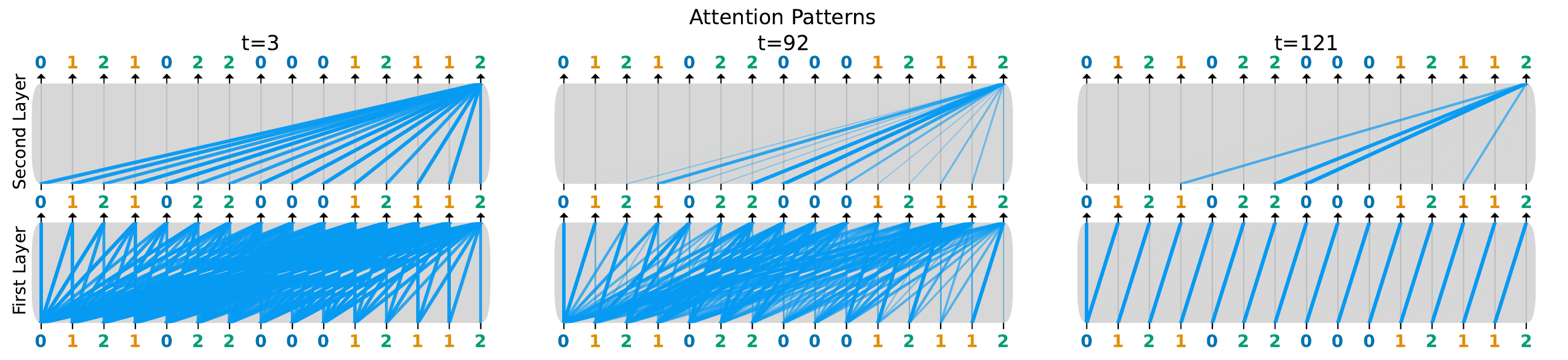}
    \caption{Attention for a fixed input at various time steps in training. These diagrams show where the attention heads are attending to at each layer. In the second layer, only the last token attention is shown. Tokens on top attend to tokens below them.
    Attention starts off uniform, but by the end of training, the layers are clearly acting the same as the induction head construction. Specifically, in the first layer each token is attending to the previous token. In the second layer, the current token, a $2$, is attending to tokens that followed $2$s, allowing bigram statistics to be calculated.
    Figure~\ref{fig:attn_heatmap} shows the full attention matrices as heatmaps.}
    \label{fig:attn}
\end{figure*}

\subsection{Related Work}
\paragraph{In-context Learning.} 
Recently, many works have focused on understanding how ICL emerges in language models. In \citep{chan2022data}, the authors discuss how properties of the data distribution promote ICL, with a focus on empirical observations.
\citet{Xie+22} studies a data model similar to ours, demonstrating that language models trained on Hidden Markov Models (HMMs) can learn in-context HMMs not found in the training data.
\cite{Abe+23} study the ability of transformers to segment the context into pairs of examples and labels and provide learning guarantees when the labeling is of the form of a sparse function.
The work of \cite{bietti2023birth} studies the dynamics of training transformers on a task that is reminiscent of our Markov chain setting but has additional complexities. Instead of drawing a fresh Markov chain for each sequence, in their task all sequences are sampled from the same Markov chain; after certain `trigger' tokens, the following `output' token is chosen deterministically within a sequence. Thus, successful prediction requires incorporating both global bigram statistics and in-context deterministic bigram copying, unlike in our setting where the patterns computed by \emph{statistical} induction heads are necessary and sufficient. As in our work, they identify multiple distinct stages of training and show how multiple top-down gradient steps lead to a solution.

Other works observe that ICL is possible due to the ability of transformers to implement gradient descent as a ``meta learning'' algorithm, and show some evidence that this indeed corresponds to how transformers learn in-context \citep{von2023transformers, dai2022can}. The work of \citet{li2023transformers} presents a theoretical framework for studying ICL, providing some risk bounds on ICL of supervised learning algorithms. \citet{guo2023transformers} construct synthetic in-context learning problems with a compositional structure, studying the representation capacity of transformers to learn these problems in-context. In \citep{hendel2023context}, the authors demonstrate that transformers learn to represent task vectors, providing a mechanistic analysis of ICL in LLMs. \citet{Kir+22} view ICL as a broad meta-learning paradigm, and observe that transformers meta-trained on real image classification tasks undergo similar phase transitions as the ones we observe in this work.

\paragraph{Induction Heads.} \citet{elhage2021mathematical} studies the formation of induction heads, sub-components of transformers that match previous occurrences of the current token, retrieving the token that succeeds the most recent occurrence. \citet{olsson2022context} studies in-context learning by analyzing the formation of induction heads in language models, showing empirical evidence that both large and small transformers display a phase transition in the ability to learn in-context. \citet{reddy2023mechanistic} also studies the formation of induction heads and their role in ICL, showing empirically that a three layer network exhibits a sudden formation of induction heads towards solving some ICL problem of interest.  \citet{bietti2023birth} study the effect of specific trigger tokens on the formation of induction heads.

\paragraph{Phase Transitions.} It has been observed in different contexts that neural networks and language models display a sudden drop in loss during their training process. This phase transition is often related to emergence of new capabilities in the network. The work of \citep{power2022grokking} observed the ``grokking'' phenomena, where the test loss of neural networks sharply drops, long after the network overfits the training data. \citep{chen2023sudden} shows another example of a phase transition in language model training, where the formation of specific attention mechanisms happen suddenly in training, causing the loss to quickly drop. \citet{barak2022hidden} observe that neural networks trained on complex learning problems display a phase transition when converging to the correct solution. Several works \citep{Kum+23,Lyu+23} attribute these phase transitions to rapid changes in the inductive bias of networks, while \citet{Mer+23} argue that the models are sparser after the phase change. \cite{SMK23} warn that phenomena in deep learning that seem to be discontinuous can actually be understood to evolve continuously once seen through the right lens.

\paragraph{Simplicity Bias.} Various works observed that neural networks have a ``simplicity bias'', which causes them to ``prioritize'' learning simple patterns first \citep{arpit2017closer, valle2018deep}. The work of \citep{kalimeris2019sgd} shows that SGD learns functions of increased complexity, first fitting a linear concept to the data before moving to more complex functions. \citep{shah2020pitfalls} shows that the simplicity bias of neural networks can sometimes be harmful, causing them to ignore important features of the data. \citet{chen2023sudden} demonstrate the effect of simplicity bias on language tasks that require understanding of syntactic structure. \citet{abbe2023sgd} provide a theoretical framework for understanding how the simplicity of the target function can govern the convergence time of SGD, describing how simple partial solutions can speed up learning; in contrast, in our setting, the unigram solution appears likely to be a distractor which delays learning of the correct solution.

\paragraph{Concurrent works} In parallel to this work, there have been a number of papers devoted to the study of similar questions regarding in-context learning and Markov Chains: \citet{Aky+24} empirically compare the ability of different architectures to perform in-context learning of regular languages. \citet{Hoo+24} observe similar stage-wise learning behaviors on transformers trained on language or synthetic linear regression tasks. \citet{Mak+24} study the loss landscape of transformers trained on sequences sampled from a single Markov Chain.



\section{Setup}
In this section, we describe our learning problem, several analytical properties of it, and present the neural networks that we will use for learning.
\subsection{ICL-MC Task}
Our learning task consists of Markov Chains with random transition matrices. The goal is to in-context estimate the transition probabilities from sampled sequences, in order to predict the next state.
Formally, a sample is a Markov Chain with state space $S = \left\{ 1, \ldots, k\right\}$ and a transition matrix $\mathcal{P}$ randomly sampled from some prior distribution, with $x_1$ drawn from some other prior distribution (potentially dependent on $\mathcal{P}$), and the rest of $\bm{x} = \left( x_1, \ldots, x_t \right)$ drawn from the Markov Chain. We primarily focus on the case where each row of the matrix is sampled from the Dirichlet distribution with concentration parameter $\bm{\alpha}$, i.e. $\mathcal{P}_{i, :} \sim \mathrm{Dir}(\bm{\alpha})$. We want to learn a predictor that, given context $x_1,\ldots, x_t$, predicts the next token, $x_{t+1}$. Note that this is an inherently non-deterministic task, even provided full information about the transition matrix, and as such it can better capture certain properties of language than previous in-context learning modeling approaches \citep{Gar+22}.

We focus on the case of $\bm\alpha = (1, \ldots, 1)^\top$ that corresponds to uniformly random transition probabilities between states. We draw the initial state $x_1$ from the stationary distribution $\bm \pi$ of the chain (which exists almost surely). We primarily consider the case where the number of states $k$ is 2 or 3.

In subsection~\ref{subsec:ngrams}, we consider the generalization of this setting to $n$-grams for $n > 2$. Instead of $\Pr(x_{t})$ being determined by $x_{t-1}$, we let $\Pr(x_{t})$ be determined by $x_{t-n+1}, \dots, x_{t-1}$, according to a conditional distribution $\mathcal{P}$ drawn from some prior. In particular, for each tuple of $n-1$ tokens, we sample the vector of conditional probabilities for the next state from a uniform Dirichlet distribution.

\subsection{Potential Strategies to Solve ICL-MC}
We adopt the Bayesian interpretation of in-context learning \citep{Xie+22}, in which a prior distribution is given by the training data, and, at test time, the model updates this prior given the in-context sequence. In this framework, we focus on two strategies for Bayesian inference: a \emph{unigram} strategy which assumes tokens in each sequence are i.i.d. samples, and the \emph{bigram} strategy which correctly takes into account dependencies among adjacent tokens.

\paragraph{1st strategy: Unigrams}

Since we let the Markov chain reach its stationary distribution (which exists a.s.), the optimal strategy across unigrams is just to count frequency of states and form a posterior belief about the stationary distribution. Unfortunately, the stationary distribution of this random Markov chain does not admit a simple analytical characterization when there is a finite number of states, but it can be estimated approximately. At the limit of $k \to \infty$, the stationary distribution converges to the uniform distribution \citep{BCC08}.

\paragraph{2nd strategy: Bigrams}

For any pair of states $i$ and $j$, let $\mathcal{P}_{ij}$ be probability of transitioning from $i$ to $j$. On each sample $\bm{x}$, we can focus on the transitions from the $i$-th state, which follow a categorical distribution with probabilities equal to $\left( \mathcal{P}_{i1}, \ldots, \mathcal{P}_{ik} \right)$. If we observe the in-context empirical counts $\{c_{ij}\}_{j=1}^k$ of the transitions, then $\mathcal{P}_{ij}$ is given by:
\begin{equation}
    \left( \mathcal{P}_{i1}, \ldots, \mathcal{P}_{ik} \right) \vert \bm{x} \sim \mathrm{Dir}(k, c_{i1} + \alpha_1, \ldots, c_{ik} + \alpha_k),
\end{equation}
where, recall, $\alpha_1, \ldots, \alpha_k$ are the Dirichlet concentration parameters of the prior.
Hence, each $\mathcal{P}_{ij}$ has a (marginal) distribution that is actually a Beta distribution:
\begin{equation}
    \mathcal{P}_{ij} \vert \bm{x} \sim \mathrm{Beta}\left( c_{ij} + \alpha_j, \sum_{j} \alpha_j + N_i - \alpha_j - c_{ij} \right),
\end{equation}
where $N_i$ is the total number of observed transitions from state $i$. As such, our best (point) estimate for each state $j$ is given by:
\begin{equation}
    \mathbb{E} \left[\mathcal{P}_{ij} \vert \bm{x} \right] = \frac{c_{ij} + \alpha_j}{N + \sum_i \alpha_i}.
\end{equation}
For the uniform Dirichlet, $\bm\alpha = (1, \ldots, 1)^\top$, it is $\mathbb{E} \left[\mathcal{P}_{ij} \vert \bm{x} \right] = \frac{c_{ij} + 1}{N_i + k}$.
\begin{remark}
    The bigram strategy implicitly assumes that the first token $x_1$ is sampled uniformly, as opposed to being sampled from the stationary distribution (which is used in our experiments and theoretical results). As the context length grows, the bigram statistics approach the Bayes optimal solution either way and this difference becomes negligible.
\end{remark}

\subsection{Architectures: Transformers and Simplifications}\label{ssec:models}
We are mainly interested in investigating how transformers \citep{Vas+17} can succeed in in-context learning this task. We focus on attention-only transformers with 2 layers with causal masking which is a popular architecture for language modeling. Given an input sequence $\bm{x}$, the output of an $n$-layer attention-only transformer is:
\begin{equation}\label{eq:tf-def}
    TF(\bm{x}) = P \circ (Attn_1+I) \dots \circ (Attn_n+I)
\end{equation}
Where $P\in \mathbb{R}^{k\times d}$ is a linear projection, $\mathbf{e}_{\bm{x}}\in \mathbb{R}^{t\times d}$ is an embedding of $\bm{x}$, and $Attn(\bm{x})$ is masked self attention with relative position embeddings \citep{Sha+18}, which is parameterized by $W_Q, W_K, W_V\in \mathbb{R}^{k\times d}, v\in \mathbb{R}^{t\times d}$:
\begin{equation}\label{eq:attn-def}
\begin{split}
    Attn(\mathbf{e}) &= \text{softmax}(\text{mask}(A))\mathbf{e} W_V\\
    A_{i,j} &= \frac{(\mathbf{e}_i W_Q+v_{i-j+1})(\mathbf{e}_j W_K)^\top}{\sqrt{d}}.
\end{split}
\end{equation}

During training, we minimize this loss:
\begin{equation}\label{eq:train_loss}
    L(\theta) = \underset{\substack{\bm{x} \sim \mathcal{P}\\ \mathcal{P} \sim \mathrm{Dir}(\bm{\alpha})^{\otimes k}}}{\mathbb{E}} \left[ \frac{1}{t} \sum_{p = 1}^t l\left(TF(\bm{x}; \theta)_p, x_{p+1}\right) \right],
\end{equation}
where $\theta$ denotes the parameters of the model and $l$ is a loss function such as the cross entropy or margin loss. For our experiments, we run on the standard cross-entropy loss. For our theoretical results, we analyze training under the margin loss that is a generalization of the hinge loss to the case of more than 2 classes: for hyperparameter $\Delta > 0$,
\begin{equation}\label{eq:margin_loss}
        l_M(f(e)_{p, :}, x_{p+1}) = \frac{1}{k} \sum_{\substack{i = 1,\\i \neq x_{p+1}}}^k \max \left\{ 0, \Delta + f(e)_{p, i} - f(e)_{p, x_{p+1}} \right\}.
\end{equation} 

We now show how a two-layer transformer can represent the optimal bigrams solution.


\begin{proposition}[Transformer Construction]\label{prop:transformer_construction}
A single-head two layer attention-only transformer can find the bigram statistics in the in-context learning markov chain task.
\end{proposition}
\begin{proof}
Set the internal dimension $d=3k$, and choose $\mathbf{e}_{\bm{x}}$ to be one-hot embeddings, that is, $\mathbf{e}_{\bm{x}_i}=\delta_{\bm{x}_i}$, where $\delta$ is the Kronecker delta. We will call the parameters of attention layer $i$, $W_Q^{(i)}, W_K^{(i)}, W_V^{(i)}, v^{(i)}$. Let 
\begin{align*}
    v^{(1)}&=\delta_2 \mathbf{1}^\top & W_Q^{(1)}&=\mathbf{0} &W_K^{(1)}&=\begin{pmatrix}
    cI^{k\times k} &\mathbf{0}&\mathbf{0}
\end{pmatrix}&
    W_V^{(1)}&=\begin{pmatrix}
    \mathbf{0} & I^{k\times k}&\mathbf{0}.
\end{pmatrix}
\end{align*}
The attention of this first layer will focus on the previous token, $A_{i,i-1}=c$ and for all $j\neq i-1$, $A_{i,j}=\mathbf{0}$. Then, as $c$ approaches infinity, the first attention layer at index $i$ approaches $\begin{pmatrix}
    \mathbf{0} & \mathbf{e}_{\bm{x}_{i-1}} & \mathbf{0}
\end{pmatrix}$.
\begin{align*}
    v^{(2)}&=\mathbf{0}&
    W_Q^{(2)}&=\begin{pmatrix}
       c I^{k\times k}&\mathbf{0}&\mathbf{0}
    \end{pmatrix}
   & W_K^{(2)}&=\begin{pmatrix}
    \mathbf{0}&I^{k\times k}&\mathbf{0}
\end{pmatrix}&
    W_V^{(2)}&=\begin{pmatrix}
     \mathbf{0}&\mathbf{0}&I^{k\times k}.
\end{pmatrix}
\end{align*}
The second layer attention counts the bigrams of past tokens, as $c$ approaches infinity, at index $i,j$ the attention is (a constant times) a count of how many positions $p< i$, $\bm{x}_p = j$ and $\bm{x}_{p-1}=\bm{x}_i$. Letting $P=\begin{pmatrix}
     \mathbf{0} &\mathbf{0} &I^{k\times k}
\end{pmatrix}$, the output approaches the probabilities given by empirical bigram statistics.\footnote{Technically, the output of this is probabilities, not log probabilities as generally cross-entropy loss assumes. This can be approximated linearly by setting
$P=\begin{pmatrix}b\mathbf{1}^\top\mathbf{1} &\mathbf{0} &aI^{k\times k}\end{pmatrix}$ to change the output from $x$ to $ax+b$. In practice, this approximation can achieve close to Bayes optimal loss.}
\end{proof}

\paragraph{Simplified Transformer Architecture.} As we see from the construction, there are two main ingredients in the solution realized by the transformer;  (1st layer) the ability to look one token back and (2nd layer) the ability to attend to itself. For this reason, we define a \textit{minimal model} that is expressive enough to be able to represent such a solution, but also simple enough to be amenable to analysis. Let $e_{x_p}$ denote the one-hot embedding that corresponds to the state at position $p\in [t],$ and let $E$ be the $\mathbb{R}^{t \times k}$ one-hot embedding matrix. Then the model is defined as:
\begin{equation}\label{eq:min-mod-def}
    f(E) = \text{mask}\left(E W_k (M E)^T\right) E, \text{ where } M  = \begin{pmatrix} v_1 & 0 & \ldots & 0 \\ v_2 & v_1 & \ldots & 0 \\ \vdots & \vdots & \cdots & \vdots \\ v_t & v_{t-1} & \ldots & v_1 \end{pmatrix} \in \mathbb{R}^{t \times t} \text{ and } W_k \in \mathbb{R}^{k \times k}.
\end{equation}
Here $\text{mask} \left( \cdot \right)$ is a causal mask. Notice that the role of $W_k$ is to mimic the attention mechanism of the second layer and the role of $v$ is that of the positional embeddings. 

\begin{fact}
Both the bigrams strategy and the unigrams strategy can be expressed by the minimal model with a simple choice of weights. 
\begin{itemize}
    \item \textit{Bigrams:} For $v = (0, 1, 0 \ldots, 0)^\top$, $W_k = I_{k \times k}$, we have $f(E)_{p, i} = \sum_{t^\prime = 2}^{p} \mathds{1} \left\{x_{t^\prime} = i\right\} \mathds{1} \left\{x_{t^\prime - 1} = x_p\right\}$.
    \item \textit{Unigrams:} For $v = (1, 0, 0 \ldots, 0)^\top$, $W_k = 11^T$, we have $f(E)_{p, i} = \sum_{t^\prime = 1}^{p} \mathds{1} \left\{x_{t^\prime} = i\right\}$.
\end{itemize}
\end{fact}


\begin{figure*}[t!]
    \includegraphics[width=0.65\textwidth]{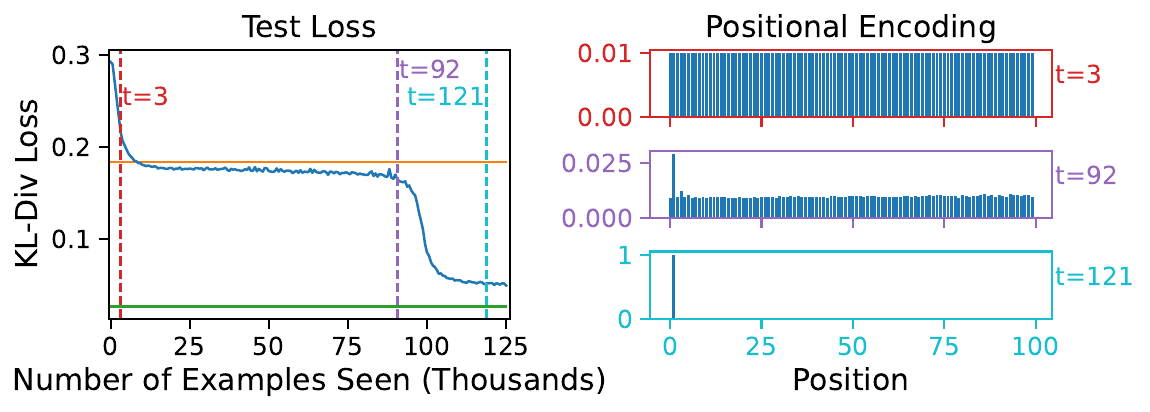}
    \includegraphics[width=0.36\textwidth]{figures/transformer_toy/3symb_similarity.pdf}
    \includegraphics[width=0.65\textwidth]{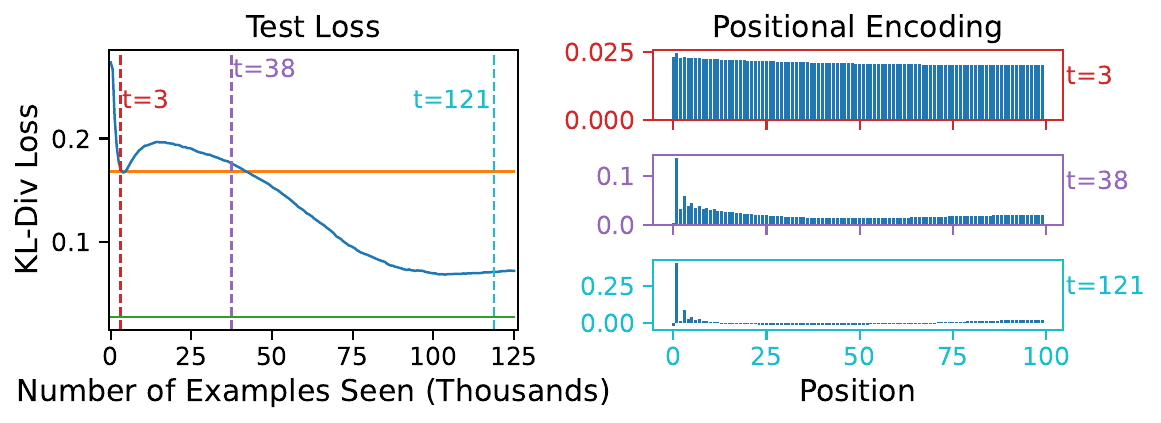}
    \includegraphics[width=0.36\textwidth]{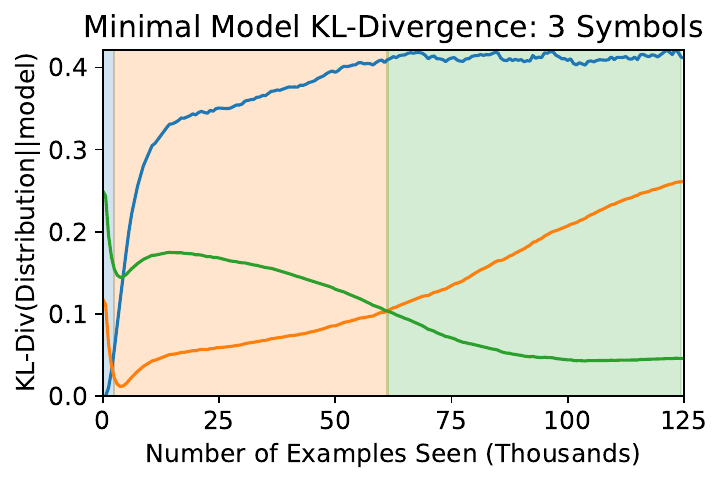}
    \caption{ A two layer transformer (\textbf{top}) and a minimal model (\textbf{bottom}) trained on our in-context Markov Chain task. A comparison of the two layer attention-only transformer and minimal model (\ref{eq:min-mod-def}) (with $v$ having constant uniform initialization, and $W_K$ initialized to $0$). The graphs on the left are test loss measured by KL-Divergence from the underlying truth. The green line shows the loss of the unigram strategy, and the orange line shows the loss of the bigram strategy. The middle graph shows the effective positional encoding (for the transformer, these are for the first layer, and averaged over all tokens).
    The graph on the right shows the KL-divergence between the outputs of the models and three strategy. The lower the KL-divergence, the more similar the model is to that strategy. 
    }
    \label{fig:3symb_pos}
\end{figure*}
\section{Empirical Findings and Theoretical Validation}


In this section, we present our empirical findings on how transformers succeed in in-context learning Markov Chains, we demonstrate the different learning stages during training and the sudden transitions between them, and draw analytical and empirical insights from the minimal model. 

\subsection{Transformers in-context learn Markov Chains hierarchically}

As can be seen in Figure~\ref{fig:3symb_pos}, all the models converge near the Bayesian optimal solution, suggesting that they learn to implement the bigram strategy. Curiously, however, the learning seems to be happening in stages; there is an initial rapid drop and the model quickly finds a better than random solution. Afterwards, there is a long period of only slight improvement before a second rapid drop brings the model close to the Bayes optimal loss. We observe that training a 1-layer transformer fails to undergo a phase transition or converge to the right solution - see Figure \ref{fig:one_layer}.

Interestingly, as can be seen from the horizontal lines in Figure~\ref{fig:3symb_pos}, the intermediate plateau corresponds to a phase when the model reaches the unigram baseline. We provide evidence that this is not a coincidence, and that after the initial drop in loss, the model's strategy is very similar to the unigram strategy, before eventually being overtaken by the bigram strategy. Some of the strongest such evidence is on the right in Figure \ref{fig:3symb_pos}, where we plot the KL divergence between model's prediction and the two different strategies. For both the strategies, their KL divergence from the model quickly goes down, with the unigram solution being significantly lower. Around the point of the second loss drop, the KL divergence between the model and the bigram solution decreases, while the other one increases, making it clear that the model transitions from the one solution to the other. This final drop is what has been associated to prior work with \textit{induction heads} formation \citep{olsson2022context}; special dedicated heads inside a transformer are suddenly being formed to facilitate in-context learning.


\paragraph{Mechanistic Evidence For Solutions Found By Transformer.}
To confirm how the two layer attention-only transformer solves ICL-MC, we inspected the attention in each layer throughout training. Figure \ref{fig:attn} shows the attention for a particular input during different parts of training. By the end of training, the attention consistently matches that of our construction, with the first layer attending to tokens one in the past, and the second layer attending to tokens that follow the same token as the current one. 
Note that even if the second layer of the transformer is mostly the same as at the end of training, if the first layer is different, then the weights shown for the second layer attention could differ dramatically. See also Figure \ref{fig:tf_loss} in the Appendix that displays how the models perform on different parts of the distribution during training.

\paragraph{Varying the data distribution - Unigrams slow down learning}

There are several interesting phenomena in the learning scenario that we just described, but it is the second drop (and the preceding plateau) that warrants the most investigation. In particular, one can ask the question: is the unigram solution helpful for the eventual convergence of the model, or is it perhaps just a by-product of the learning procedure?

To answer these questions, we define distributions over Markov chains that are in between the distribution where unigrams is Bayes optimal, and the distribution where unigrams is as good as uniform. As we see in Figure \ref{fig:interpolate}, the transformers that are being trained on the distribution where there is no unigrams ``signal" train much faster.

\subsection{Theoretical Insights from the Minimal Model}

To abstract away some of the many complicated components from the transformer architecture, we focus our attention now to the minimal model of Section \ref{ssec:models}. We train minimal models of eq.~\eqref{eq:min-mod-def}, starting from a deterministic constant initialization, by minimizing the cross entropy loss with sgd. Full experimental details can be found in the Appendix. Figure \ref{fig:3symb_pos} (bottom) displays the training curves for the minimal model. Similar to the transformer, the model learns to converge to the bigrams solution, spending however significantly less time, if any, to the unigram solution - even though they can represent it.

We now provide theoretical insights on how training progresses stage by stage and how this is achieved by the synergy between the two layers. As it turns out, there need to be at least two steps of gradient descent in order for both elements of the solution to be formed. The following lemma quantifies this.

\begin{lemma}\label{lem:min_model}
    Let the model defined as in eq.~\eqref{eq:model} and initialized with $W_k = c 11^T, v = c 1^T$. Then, the after one step of stochastic gradient descent on the margin loss of eq.~\eqref{eq:margin_loss} we have:
    \begin{equation*}
            W_k^{(1)} = \begin{pmatrix}c & c \\ c & c\end{pmatrix} + c \eta \left[ O(t^2) \begin{pmatrix}B & A \\ A & B\end{pmatrix}  + O(t) \right] \qquad
            v_j^{(1)}  = c + \frac{c \eta}{t} \left[ \frac{(t - j + 1)(t-j+2)}{2} D  + O(t) \right], j \in [t]
    \end{equation*}
    where $A, B, D > 0$ with $B \approx 4A$ (diagonal bias) and $\eta$ is the learning rate. After the second step, $v_2^{(2)}$ becomes dominant, i.e. $v_2^{(2)} > v_j^{(2)}, j = 1, 3, 4, \ldots, t$.
\end{lemma}


See Appendix \ref{sec:proofs} for the proof. We see that the gradient of $W_k$ at initialization has a clear diagonal bias, and so it starts driving the 2nd layer towards the correct solution - see Constructions in subsection \ref{ssec:models}. The gradient of the positional embeddings, $v$, however, is rather ``uninformative" to the correct coordinate that contains the one back position. Instead, it mostly favors the first position (look at current token), and has a quadratic decay. It is only after the second step that the 1st layer also starts realizing the solution, by growing $v_2^{(2)}$ - see Figure \ref{fig:min_model_margin}.

There are several interesting observations that follow from Lemma \ref{lem:min_model} and can help us understand better the two stages of learning - both in the toy model and in the transformer:

\paragraph{2nd layer is learnt first}

It has been observed before in a similar bigram learning setting with a two-layer transformer that the model might be learning first the 2nd layer \citep{bietti2023birth}. We also make similar observations in our experiments with the minimal model and the transformers \ref{fig:attn}, although the evidence is not so clear in the latter case. For the minimal model, the gradient calculations, clearly suggest that starting from a default initialization, it is only the 2nd layer that quickly ``picks up" the right solution.

\paragraph{Passing by unigrams.} The gradient of the positional embeddings $v$ at initialization delivers no information about the optimal solution, but has a quadratic structure instead that mostly favors $v_1$. If the gradient of the 2nd layer is small in scale (which can happen either due to the initialization of $v$ or to a small learning rate), then $W_k$ will not deviate much from the uniform initialization. Notice, then, that a uniform $W_k$ together with a biased first coordinate of $v$ correspond precisely to the unigram solution in this model. Thus, it is clear in this case that $W_k$ not being aligned with the solution yet, not only slows down $v$, but also biases it towards a simpler solution. We believe that the same mechanism, perhaps manifested differently, is also responsible for the hierarchical learning in the transformer. 

\paragraph{Even/odd pattern.} An interesting observation that comes out from the calculations is the form of $v_j$ after 2 steps; the gradient amplifies much more the even coordinates than the odd ones, with a scale that follows a geometric series. In fact, the ratio is closely related to the moments of the eigenvalue of the transition matrix. This way the second coordinate starts growing larger than the rest and the model eventually learns to represent the bigrams solution. As a byproduct, however, the rest of the even coordinates also grow in magnitude, despite not being part of the optimal solution. Perhaps surprisingly, we are able to also identify the same spurious pattern in the transformers! In the transformer, if the positional embeddings are dominating attention (that is, if $W_Q$ is unimportant in that layer), we can define an analogue of the minimal model: $\hat{v}_i = \text{softmax}(e_{x_i}W_k v^T)$. Unlike $v$ in the minimal model, $\hat{v}_i$ depends on the token at position $i$, but in practice, during most of training, $\hat{v}_i$ is similar regardless of the value of $x_i$. As can be seen on the top of Figure~\ref{fig:3symb_pos} (see also Figure~\ref{fig:2symb_pos}), shortly before the second drop in the loss as the induction head forms around $t=92$, the positional embeddings start showing the same even/odd pattern for the first positions.

\begin{figure*}[t!]
\centering
\begin{subfigure}{0.37\textwidth}
\includegraphics[width=\textwidth]{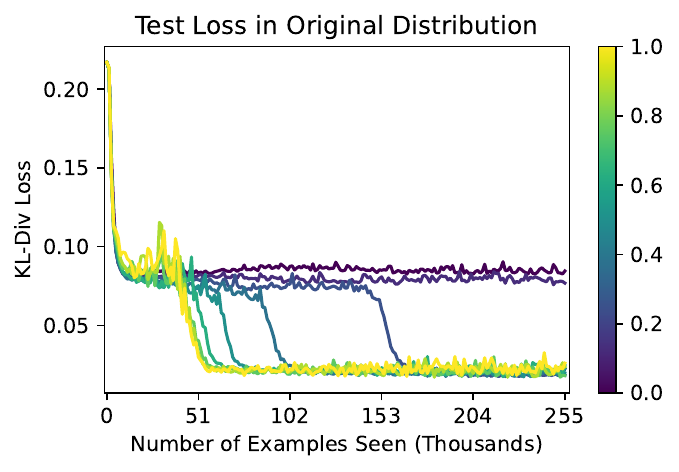}
\caption{} \label{fig:interpolate}
\end{subfigure}
\hspace{0.01\textwidth}
\begin{subfigure}{0.6\textwidth}
   \includegraphics[width=\textwidth]{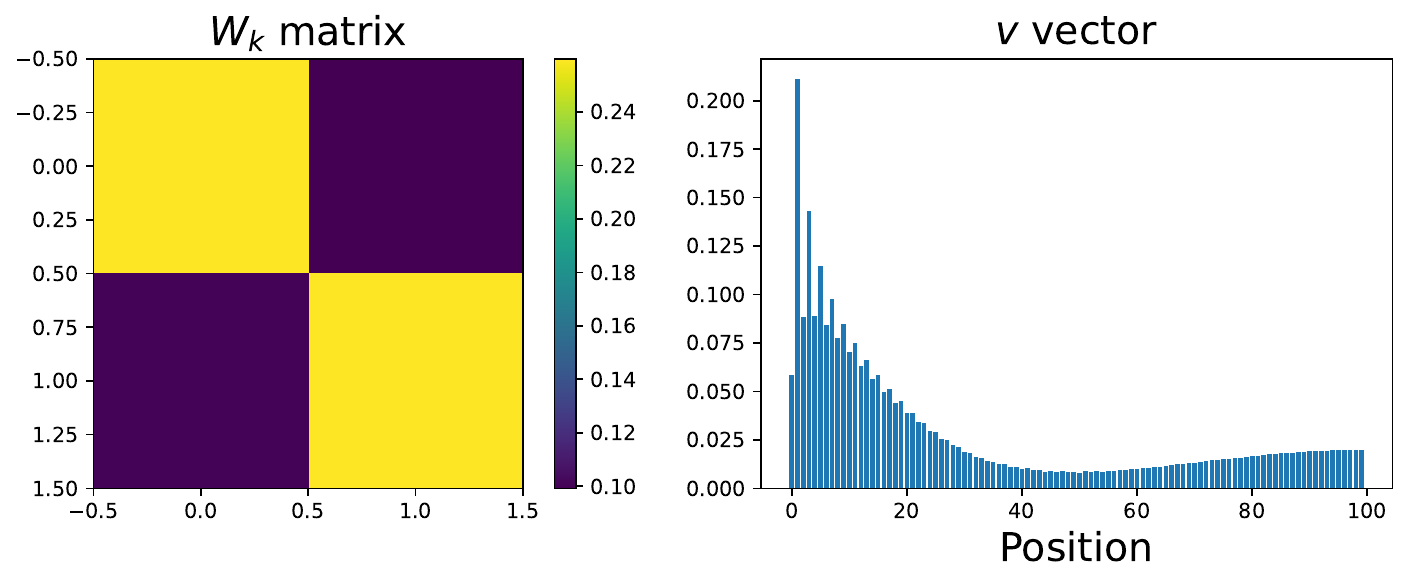} 
   \caption{}\label{fig:min_model_margin}
\end{subfigure}

\caption{(\textbf{a}) Transformers trained on different data distributions and evaluated at the original distribution. Color displays a smooth interpolation between data distributions of uninformative unigrams strategy (purple, 0) and Bayes optimal unigrams (yellow, 1). When there is not much signal from unigrams, learning progresses faster without long plateaus. See in Appendix \ref{para:interpolate} for a description of the data distributions. (\textbf{b}) Training of the minimal model in In-Context Learning Markov Chains with $k=2$ states. (\textbf{left}) The heatmap of the 2nd layer ($W_k$ matrix) that learns to be close to diagonal. (\textbf{right}) The values of the positional embeddings (1st layer) that display the curious even/odd pattern. Timestep corresponds to a phase when the model has started implementing the bigrams solution, but has not converged yet.}
    
\end{figure*}

\begin{figure*}[t!]
\centering
    \includegraphics[width=0.49\textwidth]{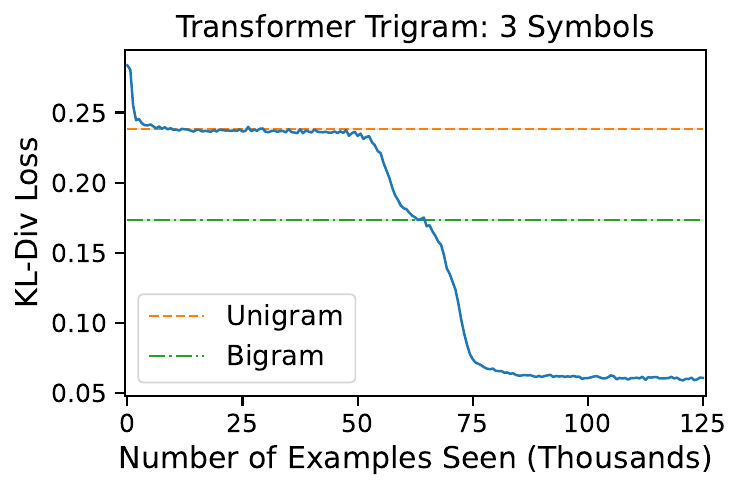}
    \includegraphics[width=0.49\textwidth]{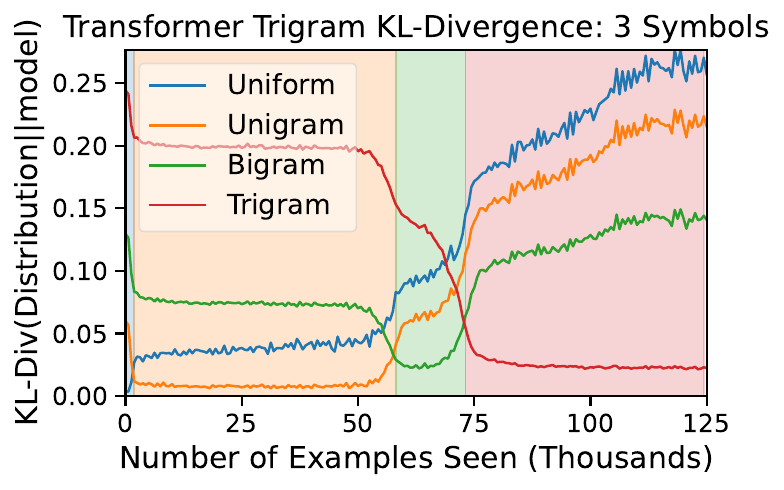}
    \caption{Three-headed transformer trained on In-Context Learning 3-grams (trigrams), with context length 200. \textbf{Left:} Loss during training. The model hierarchically converges close to the Bayes optimal solution. \textbf{Right:} KL divergence between the model and different strategies during training. As we observe, there are 4 stages of learning, each of them corresponding to a different algorithm implemented by the model.}\label{fig:tri}  
\end{figure*}
\paragraph{If unigram solution is not correlated, it is skipped.} By tweaking the data distribution, we can probe different parts of learning. In particular, when we focus on a special case of distribution where there is no unigrams solution, transformers learn much faster and do not plateau for long in the intermediate stage - see Figure \ref{fig:interpolate}. The calculations also seem to confirm this (Corollary \ref{cor:aequalb} in the Appendix): the gradient of $v$ at the first step evaluated on this special distribution shows that, in contrast to the default case, there is no ``heavy'' preference towards the first coordinate. Then, once there is enough signal from the second layer, training progresses much more smoothly.

We now show that two-stage learning is not only necessary in some sense, but also sufficient to reach the bigrams solution. We analyze the optimization process of the minimal model and discover, that much like in the experiments with the transformers, there are separate phases that correspond to learning of different parts of the model. In particular, our analysis is in a simplified setting where the data distribution changes in the second step; the transition matrices are no longer sampled uniformly, but they are instead the ``hard'' examples of this task and contain the most ``information'' about the solution. This idea is closely related to the technique of curriculum learning \citep{Ben+09} and is naturally motivated in many settings. In this setting, we are able to show that two steps of gradient descent will reach the bigrams solution.

\begin{proposition}(Informal)
Consider the minimal model of eq.~\eqref{eq:min-mod-def} being trained with online SGD on the margin loss~\eqref{eq:margin_loss} with in-context loss~\eqref{eq:train_loss}. With appropriate choice of learning rates and weight decay in the 2nd layer, two steps of gradient descent with curriculum learning reach the bigrams solution.
\end{proposition}

The idea of the proof is that a first step of gradient descent with a small learning rate can align the 2nd layer, while a second step on ``hard'' examples can learn to identify perfectly the relevant relative embedding. The formal statement and its proof can be found in Appendix \ref{sec:proofs}. An interesting by-product of the analysis is that the learning rate at the first step needs to be an order of magnitude larger than at the second step, perhaps explaining why the second phase of learning in the experiments takes significantly more time. It is worth noting that, while this is a simplified setting, it goes beyond NTK-based \citep{JHG18} analyses where the representations do not change much and it crucially involves more than one step which has been a standard tool in the analysis of feature learning in deep learning \citep{Ba+22}.

\subsection{Beyond Bigrams: \texorpdfstring{$n$}{n}-gram Statistics}
\label{subsec:ngrams}
We also investigated the performance of transformers on learning in-context $n$-grams for $n > 2$; in particular, 3-grams. We trained attention-only transformers with three heads in each layer by minimizing the in-context cross entropy loss with the Adam optimizer. As can be seen in Figure~\ref{fig:tri} (left), the model eventually converges to the Bayes optimal solution. Interestingly, as in the case of Markov Chains, the model displays a ``hierarchical learning" behavior characterized by long plateaus and sudden drops. In this setup, the different strategies correspond to unigrams, bigrams and trigrams, respectively. This is presented clearly on the right of Figure~\ref{fig:tri}, where we plot the similarity of the model with the different strategies and it exhibits the same clear pattern as in the case of $n=2$. 
Single attention headed models could not achieve better performance than bigrams. We leave a more thorough investigation for future work. 

\section{Conclusion and Discussion}
In this work, we have introduced a simple learning problem which serves as a controlled setting for understanding in-context learning and the emergence of (statistical) induction heads. Induction heads are a common motif in LLMs, which suggests the natural direction of exploring to what extent our various findings extend to natural language training data. Notably, on the way to solving the ICL-MC task, networks pass through a sequence of well-characterized discrete incomplete solutions. Simple but incomplete solutions may be commonplace in language modeling and other rich learning settings; for any such solution, one can ask to what extent its presence speeds up or slows down the formation of more complex circuits with higher accuracy.

\paragraph{Acknowledgments.} BE acknowledges funding from the ONR under award N00014-22-1-2377 and the NSF under award IIS 2229881. NT acknowledges support through the National Science Foundation under NSF Award 1922658. NT would like to thank Boaz Barak, Cengiz Pehlevan and the whole ML Foundations Group at Harvard for their hospitality during Fall 2023 when most of this work was done.

\bibliography{paper}
\bibliographystyle{apalike}

\renewcommand{\theequation}{\thesection.\arabic{equation}}

\newpage
\appendix
\section{Proofs}\label{sec:proofs}
In this section, we present our theoretical results on in-context learning Markov Chains with the minimal model of Section \ref{ssec:models}.

\paragraph{Setup and notation} Our data consists of sequences of length $t$, $\bm{x} = \left( x_1, \ldots, x_t \right)$, drawn from a Markov Chain with state space $S = \left\{ 1, \ldots, k\right\}$ (i.e., $x_{j} \in \{1, \ldots, k\}$ for all $j \in [t]$), and a random transition matrix $\mathcal{P}$. Each row of the matrix is sampled from a Dirichlet distribution with concentration parameter $\bm{\alpha}$, i.e. $\mathcal{P}_{i, :} \sim \mathrm{Dir}(\bm{\alpha})$. Unless stated otherwise, we set $\bm\alpha = (1, \ldots, 1)^\top$, corresponding drawing the row from a uniform distribution over the simplex. Let $e_{x_p} \in \{0, 1\}^k$ denote the one-hot embedding of the state at position $p \in [t]$ and let $e \in \mathbb{R}^{t \times k}$ be the embedding matrix. We assume there are 2 states in the Markov Chain, i.e. $k = 2$. In that case, the transition matrix can be parameterized as:
\begin{equation}
    \mathcal{P} = \begin{pmatrix}
        a & 1-a \\ 1-b & b
    \end{pmatrix},
\end{equation}
where $a, b \sim Unif(0,1)$.

\paragraph{Model}
We define our model as a simplified sequence to sequence linear transformer $f: \mathbb{R}^{t \times k} \to \mathbb{R}^{t \times k}$ with $f(e) = \text{mask}\left(e W_k (M e)^T\right) e$. It is $W_k \in \mathbb{R}^{k \times k}$ and $M  = \begin{pmatrix} v_1 & 0 & \ldots & 0 \\ v_2 & v_1 & \ldots & 0 \\ \vdots & \vdots & \cdots & \vdots \\ v_t & v_{t-1} & \ldots & v_1 \end{pmatrix}$, where $v = [v_1, v_2, \ldots, v_t] \in \mathbb{R}^{t}$. Equivalently, we can express the $i$-th logit for the $p$-th position as:
\begin{equation}\label{eq:model}
    f(e)_{p, i} = \sum_{t^\prime = 1}^p \mathds{1} \left\{x_{t^\prime} = i\right\} \sum_{s = 1}^{t^\prime} v_{t^\prime - s + 1} e^\top_{x_p} W_k e_{x_s}.
\end{equation}
The model can represent the unigrams and bigrams solutions as following:
\begin{itemize}
    \item Construction for bigrams: $v = (0, 1, 0, \ldots, 0)^\top$ and $W_k = I_{k \times k}$, then $f(e)_{p, i} = \sum_{t^\prime = 2}^{p} \mathds{1} \left\{x_{t^\prime} = i\right\} \mathds{1} \left\{x_{t^\prime - 1} = x_p\right\}$.
    \item Construction for unigrams: $v = (1, 0, 0, \ldots, 0)^\top$ and $W_k = 11^T$ (all ones), then $f(e)_{p, i} = \sum_{t^\prime = 1}^{p} \mathds{1} \left\{x_{t^\prime} = i\right\}$.
\end{itemize}

\paragraph{Training}

We analyze stochastic gradient descent training with the margin loss $l_M(f(e)_{p, :}, x_{p+1}) = \frac{1}{k} \sum_{i = 1, i \neq x_{p+1}}^k \max \left\{ 0, \Delta + f(e)_{p, i} - f(e)_{p, x_{p+1}} \right\}$. The total loss (sum of the losses across all positions) is given by
\begin{equation}
    L(f(e), e) = \left[ \frac{1}{t} \sum_{p = 1}^t l_M(f(e)_{p, :}, x_{p+1})\right],
\end{equation}
and the population loss:
\begin{equation}\label{eq:pop_total_loss}
    \mathcal{L} = \underset{\substack{\bm{x} \sim \mathcal{P}\\ \mathcal{P} \sim \mathrm{Dir}(\bm{\alpha})^{\otimes k}}}{\mathbb{E}} L (f(e), e),
\end{equation}
where recall $e$ depends on $\bm{x}$.

We next compute gradients of the loss for the first two steps.

\begin{lemma}
    Let the model defined as in eq.~\eqref{eq:model} and initialized with $W_k = c 11^T, v = c 1^T$. Then, after one step of gradient descent on the population loss~\eqref{eq:pop_total_loss} we have:
    \begin{equation}
        \begin{split}
            W_k^{(1)} & = \begin{pmatrix}c & c \\ c & c\end{pmatrix} + c \eta \left[ O(t^2) \begin{pmatrix}B & A \\ A & B\end{pmatrix}  + O(t) \right] \\
            v_j^{(1)} & = c + \frac{c \eta}{t} \left[ \frac{(t - j + 1)(t-j+2)}{2} D  + O(t) \right], j \in [t],
        \end{split} 
    \end{equation}
    where $A, B, D > 0$ with $B \approx 4A$ (diagonal bias) and $\eta$ is the learning rate. After the second step, $v_2^{(2)}$ becomes dominant, i.e. $v_2^{(2)} > v_j^{(2)}, j = 1, 3, 4, \ldots, t$.
\end{lemma}

\begin{proof}
\textbf{First step}. We analyze the first step of stochastic gradient descent. The function at initialization is
\begin{equation}
    \begin{split}
        f^0(e)_{p, i} & = c^2 \sum_{t^\prime = 1}^p \mathds{1} \left\{x_{t^\prime} = i\right\} \sum_{s = 1}^{t^\prime} e^\top_{x_p} 11^T e_{x_s} \\
        & = c^2 \sum_{t^\prime = 1}^p \mathds{1} \left\{x_{t^\prime} = i\right\} t^\prime \; \; \in [0, c^2 \frac{p(p+1)}{2}].
    \end{split}
\end{equation}
By choosing $\Delta \geq c^2 \frac{t(t+1)}{2}$, we ensure that $\Delta + f^0(e)_{p, i} - f^0(e)_{p, x_{p+1}} \geq 0$, for all $p \in [t]$. From the total law of expectation it is:
\begin{equation}
    \underset{\substack{\bm{x} \sim \mathcal{P}\\ \mathcal{P} \sim \mathrm{Dir}(\bm{\alpha})^{\otimes k}}}{\mathbb{E}} \left[ L \right] = \mathbb{E}_{\Pm \sim \mathrm{Dir}(\bm{\alpha})^{\otimes k}} \left[ \mathbb{E}_{\bm{x} \sim \Pm} \left[L \vert \Pm \right] \right].
\end{equation}
We will first focus on the inner conditional expectation and at the end of the calculations we will take expectation with respect the transition matrix. In what follows, unless otherwise stated, $\mathbb{E}$ and $\mathbb{P}$ will be with respect to the randomness of $\bm{x}$ conditioned on $\Pm$.

Then the gradient of the loss with respect to $W_k$ is:
\begin{equation}\label{eq:loss_grad_wk}
    \begin{split}
        \nabla_{W_k} \mathbb{E}_{\bm{x} \sim \Pm} \left[ L \vert \Pm \right] & = \frac{1}{t} \sum_{p=1}^t \frac{1}{k} \sum_{i = 1, i \neq x_{p+1}}^k \mathbb{E} \left[ \mathds{1} \left\{ \Delta + f(e)_{p, i} - f(e)_{p, x_{p+1}} \geq 0 \right\} \left( \nabla_{W_k} f(e)_{p, i} - \nabla_{W_k} f(e)_{p, x_{p+1}} \right) \right] \\
        & = \frac{1}{t} \sum_{p=1}^t \frac{1}{k} \sum_{i = 1, i \neq x_{p+1}}^k \mathbb{E} \left[ \left( \nabla_{W_k} f(e)_{p, i} - \nabla_{W_k} f(e)_{p, x_{p+1}} \right) \right]\\
        & = \frac{1}{t} \sum_{p=1}^t \frac{1}{k} \sum_{i = 1}^k \mathbb{E} \left[ \left( \nabla_{W_k} f(e)_{p, i} - \nabla_{W_k} f(e)_{p, x_{p+1}} \right) \right]\\
        & = \frac{1}{t} \sum_{p=1}^{t} \left[ \left( \frac{1}{k} \sum_{i = 1}^k \mathbb{E} \left[ \nabla_{W_k} f(e)_{p, i} \right] \right) - \mathbb{E} \left[ \nabla_{W_k} f(e)_{p, x_{p+1}} \right] \right].
    \end{split}
\end{equation}
From equation~\eqref{eq:model}, we have for the gradient of logit $i$:
\begin{equation}
    \nabla_{W_k} f(e)_{p, i} = c \sum_{t^\prime = 1}^p \mathds{1}\left\{ x_{t^\prime} = i \right\} \sum_{s=1}^{t^\prime} e_{x_p} e_{x_s}^T,
\end{equation}
or, equivalently, its elements are:
\begin{equation}
    \left(\nabla_{W_k} f(e)_{p, i}\right)_{m, l} = c \sum_{t^\prime = 1}^p \mathds{1}\left\{ x_{t^\prime} = i \right\} \sum_{s=1}^{t^\prime} \mathds{1} \{ x_p = m \} \mathds{1} \{ x_s = l \},
\end{equation}
and their expectation (with respect to $\bm{x}$ conditioned on $\Pm$) is:
\begin{equation}
    \mathbb{E} \left[\left(\nabla_{W_k} f(e)_{p, i}\right)_{m, l} \right] = c \sum_{t^\prime = 1}^p \sum_{s=1}^{t^\prime} \mathbb{P} \left[ x_s = l, x_{t^\prime} = i, x_p = m \right].
\end{equation}
Similarly, for the ground truth logit:
\begin{equation}
    \mathbb{E} \left[\left(\nabla_{W_k} f(e)_{p, x_{p+1}}\right)_{m, l} \right] = c \sum_{t^\prime = 1}^p \sum_{s=1}^{t^\prime} \mathbb{P} \left[ x_s = l, x_{t^\prime} = x_{p+1}, x_p = m \right].
\end{equation}
Thus, by substituting back to eq.~\eqref{eq:loss_grad_wk} we have : 
\begin{equation}\label{eq:gd_W_1}
    \begin{split}
        - \nabla_{W_k} \mathbb{E}_{\bm{x} \sim \Pm} \left[ L \vert \Pm \right]
        & = \frac{c}{t} \sum_{p=1}^t \sum_{t^\prime = 1}^p \sum_{s=1}^{t^\prime} \mathbb{P} \left[ x_s = l, x_{t^\prime} = x_{p+1}, x_p = m \right] - \frac{c}{t k} \sum_{i = 1}^k \sum_{p=1}^t \sum_{t^\prime = 1}^p \sum_{s=1}^{t^\prime} \mathbb{P} \left[ x_s = l, x_{t^\prime} = i, x_p = m \right] \\
        & = \frac{c}{t} \sum_{i=1}^k \sum_{p=1}^t \sum_{t^\prime = 1}^p \sum_{s=1}^{t^\prime} \mathbb{P} \left[x_{p+1} = i, x_s = l, x_{t^\prime} = i, x_p = m \right] - \frac{c}{t k} \sum_{i = 1}^k \sum_{p=1}^t \sum_{t^\prime = 1}^p \sum_{s=1}^{t^\prime} \mathbb{P} \left[ x_s = l, x_{t^\prime} = i, x_p = m \right] \\
        & = \frac{c}{t} \sum_{i=1}^k \left( \Pm_{mi} - \frac{1}{k} \right) \sum_{p=1}^t \sum_{t^\prime = 1}^p \sum_{s=1}^{t^\prime} \mathbb{P} \left[x_s = l, x_{t^\prime} = i, x_p = m \right] \\ & = \frac{c}{t} \sum_{i=1}^k \left( \Pm_{mi} - \frac{1}{k} \right) \sum_{p=1}^t \sum_{t^\prime = 1}^p \sum_{s=1}^{t^\prime} \mathbb{P} \left[ x_p = m \mid x_{t^\prime} = i \right] \mathbb{P} \left[ x_{t^\prime} = i \mid x_s = l \right] \mathbb{P} \left[ x_s = l \right]\\
        & = \frac{c}{t} \sum_{i=1}^k \left( \Pm_{mi} - \frac{1}{k} \right) \sum_{p=1}^t \sum_{t^\prime = 1}^p \sum_{s=1}^{t^\prime} \left(\mathcal{P}^{p - t^\prime}\right)_{im} \left(\mathcal{P}^{t^\prime - s}\right)_{li} \pi_l,
    \end{split}
\end{equation}
where we used the Markov property and the assumption that the chain has reached its stationary distribution $\bm{\pi} \in \mathbb{R}^{1 \times k}$, so $\mathbb{P} \left[ x_{t} = i \right] = \pi_i$ for all $t > 0, i \in [k]$. Recall that, in the case of $k = 2$, the (random) transition matrix can be parameterized as:
\begin{equation}
    \mathcal{P} =
    \begin{pmatrix}
        a & 1-a \\ 1-b & b
    \end{pmatrix},
\end{equation}
where $a, b  \sim Unif([0, 1])$. $\Pm$ is almost surely diagonalizable, and using its eigendecomposition, we can calculate its $n$-th powers
\begin{equation}
    \begin{split}
        \mathcal{P}^n & = \frac{1}{a+b-2} \begin{pmatrix}
            1 & \frac{1 - a}{b-1} \\
            1 & 1
        \end{pmatrix} \begin{pmatrix}
            1 & 0 \\
            0 & a + b -1
        \end{pmatrix}^n \begin{pmatrix}
            b-1 & a-1 \\
            1-b & b-1
        \end{pmatrix} \\
        & = \frac{1}{a + b - 2}
        \begin{pmatrix}
            (a-1) (a + b - 1)^n + b - 1 & (1-a) (a + b - 1)^n + a - 1\\
            (1-b) (a + b - 1)^n + b - 1 & (b-1) (a + b - 1)^n + a - 1
        \end{pmatrix}.
        \end{split}
\end{equation}
We observe that every element of the matrix, $\left(\mathcal{P}^n\right)_{ij}$, is of the form $\frac{\beta_{ij} \lambda^n  + \gamma_{ij}}{\lambda - 1}$, with $\lambda = a + b - 1$. It is also $\bm{\pi} = \frac{1}{\lambda - 1} \begin{pmatrix}
    b - 1 & a - 1
\end{pmatrix}$ (the normalized eigenvector that corresponds to $\lambda$). Thus, eq.~\eqref{eq:gd_W_1} can be written as
\begin{equation}\label{eq:gd_W_1_k2}
    \begin{split}
        - \nabla_{W_k} \mathbb{E}_{\bm{x} \sim \Pm} \left[ L \vert \Pm \right] & = \frac{c}{t(\lambda - 1)^2} \sum_{i = 1}^2 \left( \Pm_{mi} - \frac{1}{2} \right) \sum_{p=1}^t \sum_{t^\prime = 1}^p \sum_{s=1}^{t^\prime} \left(\beta_{im} \lambda^{p-t^\prime}  + \gamma_{im}\right) \left(\beta_{li} \lambda^{t^\prime -s}  + \gamma_{li}\right) \pi_l \\
        & = \frac{c \pi_l}{t(\lambda - 1)^2} \sum_{i = 1}^2 \left( \Pm_{mi} - \frac{1}{2} \right) \sum_{p=1}^t \sum_{t^\prime = 1}^p \sum_{s=1}^{t^\prime} \left( \underbrace{\beta_{im} \beta_{li} \lambda^{p-s}}_{(A)}  + \underbrace{\beta_{im} \gamma_{li} \lambda^{p - t^\prime}}_{(B)} + \underbrace{\beta_{li} \gamma_{im} \lambda^{t^\prime - s}}_{(C)}  + \underbrace{\gamma_{im} \gamma_{li}}_{(D)}\right).
    \end{split}
\end{equation}

We calculate the four terms separately:
\begin{equation*}
    \begin{split}
        (A) & = \sum_{p=1}^t \sum_{t^\prime = 1}^p \sum_{s=1}^{t^\prime} \beta_{im} \beta_{li} \lambda^{p-s} \\
        & = \beta_{im} \beta_{li} \sum_{p=1}^t \sum_{t^\prime = 1}^p \sum_{s=1}^{t^\prime} \lambda^{p-s} \\
        & = \beta_{im} \beta_{li} \sum_{p=1}^t \sum_{t^\prime = 1}^p \frac{\lambda^{p-1} - \lambda^{p-t^\prime - 1}}{1 - \lambda^{-1}} \\
        & = \frac{\beta_{im} \beta_{li}}{1 - \lambda^{-1}} \sum_{p=1}^t \left( p \lambda^{p-1} -  \frac{\lambda^{p-2} - \lambda^{-2}}{1 - \lambda^{-1}} \right) \\
        & = \frac{\beta_{im} \beta_{li}}{1 - \lambda^{-1}} \left( \frac{1 - (t+1) \lambda^t + t \lambda^{t+1}}{(1 - \lambda)^2} - \frac{\lambda^{-1} - \lambda^{t - 1}}{\left( 1 - \lambda^{-1} \right)(1 - \lambda)} + t \frac{\lambda^{-2}}{1 - \lambda^{-1}} \right) \\
         & = \frac{\beta_{im} \beta_{li}}{\lambda - 1} \left( \frac{\lambda - (t+1) \lambda^{t+1} + t \lambda^{t+2}}{(1 - \lambda)^2} - \frac{1 - \lambda^{t }}{\left( 1 - \lambda^{-1} \right)(1 - \lambda)} + t \frac{1}{\lambda - 1} \right). \\
         (B) & = \sum_{p=1}^t \sum_{t^\prime = 1}^p \sum_{s=1}^{t^\prime} \beta_{im} \gamma_{li} \lambda^{p-t^\prime} \\
        & = \beta_{im} \gamma_{li} \sum_{p=1}^t \sum_{t^\prime = 1}^p \sum_{s=1}^{t^\prime} \lambda^{p-t^\prime} \\
        & = \beta_{im} \gamma_{li} \sum_{p=1}^t \sum_{t^\prime = 1}^p t^\prime \lambda^{p-t^\prime} \\
        & = \beta_{im} \gamma_{li} \sum_{p=1}^t \lambda^{p-1} \frac{1 - (p+1) \lambda^{-p} + p \lambda^{-p - 1}}{(1 - \lambda^{-1})^2} \\
        & = \beta_{im} \gamma_{li} \sum_{p=1}^t \frac{\lambda^{p-1} - (p+1) \lambda^{-1} + p \lambda^{-2}}{(1 - \lambda^{-1})^2} \\
        & = \frac{\beta_{im} \gamma_{li}}{(1 - \lambda^{-1})^2} \left( \frac{1 - \lambda^t}{1 - \lambda} - \lambda^{-1} \frac{t(t+3)}{2} + \lambda^{-2} \frac{t(t+1)}{2} \right) \\
        & = \frac{\beta_{im} \gamma_{li}}{(\lambda - 1)^2} \left( \lambda^2 \frac{1 - \lambda^t}{1 - \lambda} - \lambda \frac{t(t+3)}{2} + \frac{t(t+1)}{2} \right). \\
        (C) & = \sum_{p=1}^t \sum_{t^\prime = 1}^p \sum_{s=1}^{t^\prime} \beta_{li} \gamma_{im} \lambda^{t^\prime-s} \\
        & = \beta_{li} \gamma_{im} \sum_{p=1}^t \sum_{t^\prime = 1}^p \sum_{s=1}^{t^\prime} \lambda^{t^\prime-s} \\
        & = \beta_{li} \gamma_{im} \sum_{p=1}^t \sum_{t^\prime = 1}^p \frac{\lambda^{t^\prime-1} - \lambda^{-1}}{1 - \lambda^{-1}} \\
        & = \frac{\beta_{li} \gamma_{im}}{1 - \lambda^{-1}} \sum_{p=1}^t \left( \frac{1 - \lambda^{p}}{1 - \lambda} - \lambda^{-1}p \right) \\
        & = \frac{\beta_{li} \gamma_{im}}{1 - \lambda^{-1}} \left( \frac{t}{1 - \lambda} - \frac{\lambda - \lambda^{t+1}}{(1 - \lambda)^2} - \lambda^{-1} \frac{t(t+1)}{2} \right) \\
        & = \frac{\beta_{li} \gamma_{im}}{\lambda - 1} \left( t \frac{\lambda}{1 - \lambda} - \frac{\lambda^2 - \lambda^{t+2}}{(1 - \lambda)^2} - \frac{t(t+1)}{2} \right)
    \end{split}
\end{equation*}
\begin{equation*}
    \begin{split}
        (D) & = \sum_{p=1}^t \sum_{t^\prime = 1}^p \sum_{s=1}^{t^\prime} \gamma_{im} \gamma_{li} = \gamma_{im} \gamma_{li} \frac{t(t+1)(t+2)}{6}.
    \end{split}
\end{equation*}
Plugging these expressions back into eq.~\eqref{eq:gd_W_1_k2}, we get:
\begin{equation}\label{eq:Wk_step1}
    \begin{split}
        - \nabla_{W_k} \mathbb{E}_{\bm{x} \sim \Pm} \left[ L \vert \Pm \right] = \frac{c \pi_l \left( \mathcal{P}_{m1} - \frac{1}{2} \right)}{t(\lambda-1)^2} \bigg[ & \frac{\beta_{1m} \beta_{l1} - \beta_{2m} \beta_{l2} }{\lambda - 1} \underbrace{\left( \frac{\lambda - (t+1) \lambda^{t+1} + t \lambda^{t+2}}{(1 - \lambda)^2} - \frac{1 - \lambda^{t }}{\left( 1 - \lambda^{-1} \right)(1 - \lambda)} + t \frac{1}{\lambda - 1} \right)}_{(I)} \\
        & + \frac{\beta_{1m} \gamma_{l1} - \beta_{2m} \gamma_{l2}}{(\lambda - 1)^2} \underbrace{\left( \lambda^2 \frac{1 - \lambda^t}{1 - \lambda} - \lambda \frac{t(t+3)}{2} + \frac{t(t+1)}{2} \right)}_{(II)} \\
        & + \frac{\beta_{l1} \gamma_{1m} - \beta_{l2} \gamma_{2m}}{\lambda - 1} \underbrace{\left( t \frac{\lambda}{1 - \lambda} - \frac{\lambda^2 - \lambda^{t+2}}{(1 - \lambda)^2} - \frac{t(t+1)}{2} \right)}_{(III)} \\
        & + \left(\gamma_{1m} \gamma_{l1} - \gamma_{2m} \gamma_{l2} \right) \underbrace{\frac{t(t+1)(t+2)}{6}}_{(IV)} \bigg].
    \end{split}
\end{equation}
Finally, by taking expectation over $a, b$ (the randomness of the transition matrix), we get for the update of $W_k$ with learning rate $\eta$:
\begin{equation}
    \left(W_{k}^{(1)}\right)_{m, l} = \left(W_{k}^{(0)}\right)_{m, l} - \eta \nabla_{W_k} \mathbb{E}_{a, b} \left[ \mathbb{E}_{\bm{x} \sim \Pm} \left[ L \vert \Pm \right] \right].
\end{equation}
We consider the 4 cases separately.
Case $m, l = 1, 1$:
\begin{equation}\label{eq:Wk_11}
\begin{aligned}
    \left(W_{k}^{(1)}\right)_{1, 1} = c + \frac{\eta c}{t} \mathbb{E}_{a, b} \left[ \frac{(b-1)(a-\frac{1}{2})}{(\lambda-1)^3} \left[ \frac{(a-1)(a-b)}{\lambda - 1} (I) + \frac{2(a-1)(b-1)}{(\lambda-1)^2} (II) \right.\right.\\
    \left.\left.+ \frac{2(a-1)(b-1)}{\lambda - 1} (III) + (b-1)(b-a) (IV) \right] \right]
\end{aligned}
\end{equation}
Case $m, l = 1, 2$:
\begin{equation}\label{eq:Wk_12}
\begin{aligned}
    \left(W_{k}^{(1)}\right)_{1, 2} = c + \frac{\eta c}{t} \mathbb{E}_{a, b} \left[ \frac{(a-1)(a-\frac{1}{2})}{(\lambda-1)^3} \left[ \frac{(1-b)(a-b)}{\lambda - 1} (I) + \frac{2(a-1)(b-1)}{(\lambda-1)^2} (II) \right.\right.\\
    \left.\left.- \frac{2(b-1)^2}{\lambda - 1} (III) + (b-1)(b-a) (IV) \right] \right]
\end{aligned}
\end{equation}
Case $m, l = 2, 1$:
\begin{equation}
\begin{aligned}
    \left(W_{k}^{(1)}\right)_{2, 1} = c + \frac{\eta c}{t} \mathbb{E}_{a, b} \left[ \frac{(b-1)(\frac{1}{2} - b)}{(\lambda-1)^3} \left[ \frac{(1-a)(a-b)}{\lambda - 1} (I) - \frac{2(a-1)(b-1)}{(\lambda-1)^2} (II) \right.\right.\\
    \left.\left.+ \frac{2(a-1)^2}{\lambda - 1} (III) + (a-1)(b-a) (IV) \right] \right]
\end{aligned}
\end{equation}
Case $m, l = 2, 2$:
\begin{equation}
\begin{aligned}
    \left(W_{k}^{(1)}\right)_{2, 2} = c + \frac{\eta c}{t} \mathbb{E}_{a, b} \left[ \frac{(a-1)(\frac{1}{2} - b)}{(\lambda-1)^3} \left[ \frac{(b-1)(a-b)}{\lambda - 1} (I) - \frac{2(a-1)(b-1)}{(\lambda-1)^2} (II) \right.\right.\\
    \left.\left.- \frac{2(a-1)(b-1)}{\lambda - 1} (III) + (a-1)(b-a) (IV) \right] \right]
\end{aligned}
\end{equation}
See Figure \ref{fig:grad_calc_step1} (left) for an empirical estimation for some choice of hyperparameters. We observe that the gradient is symmetric, so it is going to be $\left(W_{k}^{(1)}\right)_{2, 1} = \left(W_{k}^{(1)}\right)_{1, 2}$, and the value in the diagonal is the same, i.e. $\left(W_{k}^{(1)}\right)_{1, 1} = \left(W_{k}^{(1)}\right)_{2, 2}$.

By only focusing on the leading order terms, $(IV)$, we have:
\begin{equation}\label{eq:Wk_1st_step}
    \left(W_{k}^{(1)}\right)_{m, l} = c + c \eta  \mathbb{E}_{a, b} \left[ \frac{\pi_m \pi_l}{(\lambda - 1)} (b-a) \left( \mathcal{P}_{m1} - \frac{1}{2} \right) \right] \frac{(t+1)(t+2)}{6} + O(t).
\end{equation}
A calculation then shows that
\begin{equation}
    \begin{split}
        \mathbb{E}_{a, b} \left[ \frac{\pi_1 \pi_1}{(\lambda - 1)} (b-a) \left( \mathcal{P}_{11} - \frac{1}{2} \right) \right] & = \mathbb{E}_{a, b} \left[ \frac{\pi_2 \pi_2}{(\lambda - 1)} (b-a) \left( \mathcal{P}_{21} - \frac{1}{2} \right) \right] \\ & = \int_{0}^{1} \int_{0}^{1} \frac{(b-1)^2(b-a)(a - \frac{1}{2})}{(a+b-2)^3} da db = \frac{\ln 256 - 5}{12} \approx 0.045,
    \end{split}
\end{equation}
and
\begin{equation}
    \begin{split}
        \mathbb{E}_{a, b} \left[ \frac{\pi_1 \pi_2}{(\lambda - 1)} (b-a) \left( \mathcal{P}_{11} - \frac{1}{2} \right) \right] & = \mathbb{E}_{a, b} \left[ \frac{\pi_1 \pi_2}{(\lambda - 1)} (b-a) \left( \mathcal{P}_{21} - \frac{1}{2} \right) \right] \\& = \int_{0}^{1} \int_{0}^{1} \frac{(b-1)(a-1)(b-a)(a - \frac{1}{2})}{(a+b-2)^3} da db = \frac{7 - 10 \ln 2}{6} \approx 0.011,
    \end{split}
\end{equation}
hence the diagonal grows a constant ($\approx 4$) times more than the off-diagonal. 

Similarly, for the expected gradient of $v$ (with respect to $\bm{x}$ conditioned on $\Pm$), we have:
\begin{equation}
    \begin{split}
        \mathbb{E} \left[\frac{\partial f(e)_{p, i}}{\partial v_j} \right] & = c \mathbb{E} \left[ \sum_{t^\prime = 1}^p \mathds{1}\left\{ x_{t^\prime} = i \right\} \sum_{s=1}^{t^\prime} \delta_{j(t^\prime - s +1)} \mathds{1}\left\{ j \leq p \right\} \right] \\
        & = c \mathbb{E} \left[ \sum_{t^\prime=1}^p \mathds{1}\left\{ x_{t^\prime} = i \right\} \mathds{1}\left\{ j \leq t^\prime \right\} \mathds{1} \left\{ j \leq p \right\} \right] \\
        & = c \sum_{t^\prime = j}^p \mathbb{P} \left[ x_{t^\prime} = i \right] \mathds{1}\left\{ j \leq p \right\} \\
        & = c \pi_i \left( p - j + 1 \right) \mathds{1}\left\{ j \leq p \right\},
    \end{split}
\end{equation}
and
\begin{equation}\label{eq:v1_grad_ground}
    \begin{split}
        \mathbb{E} \left[ \frac{\partial f(e)_{p, x_{p+1}}}{\partial v_j} \right] & = c \sum_{t^\prime = j}^p \mathbb{P} \left[ x_{t^\prime} = x_{p+1} \right] \mathds{1}\left\{ j \leq p \right\} \\ 
        & = c \sum_{t^\prime = j}^p \sum_{i = 1}^k \mathbb{P} \left[ x_{t^\prime} = x_{p+1} = i \right] \mathds{1}\left\{ j \leq p \right\} \\
        & = c \sum_{i = 1}^k \sum_{t^\prime = j}^p \mathbb{P} \left[ x_{p+1} = i \mid x_{t^\prime} = i \right] \mathbb{P} \left[ x_{t^\prime} = i \right] \mathds{1}\left\{ j \leq p \right\} \\
        & = c \sum_{i = 1}^k \pi_i \sum_{t^\prime = j}^p \left( \mathcal{P}^{p+1 - t^\prime} \right)_{ii} \mathds{1}\left\{ j \leq p \right\} \\
        & = c \sum_{i = 1}^k \frac{\pi_i}{\lambda-1} \sum_{t^\prime = j}^p \left( \beta_{ii} \lambda^{p+1-t^\prime} + \gamma_{ii} \right) \mathds{1}\left\{ j \leq p \right\} \\
        & = c \left( \frac{b-1}{(\lambda-1)^2} \sum_{t^\prime = j}^p\left( \left(a-1\right) \lambda^{p+1-t^\prime} + (b-1)\right) + \frac{a-1}{(\lambda-1)^2} \sum_{t^\prime = j}^p\left( \left(b-1\right) \lambda^{p+1-t^\prime} + (a-1)\right) \right) \mathds{1}\left\{ j \leq p \right\} \\
        & = c \left( 2 \frac{(a-1)(b-1)}{(\lambda-1)^2} \frac{\lambda^{p+1-j} - 1}{1 - \lambda^{-1}} + \frac{(a-1)^2 + (b-1)^2}{(\lambda-1)^2} \left( p - j + 1 \right) \right) \mathds{1}\left\{ j \leq p \right\}.
    \end{split}
\end{equation}
Thus, the update would be:
\begin{equation}\label{eq:v_step1}
    \begin{split}
        v_j^{(1)} & = v_j^{(0)} - \eta \frac{\partial \mathcal{L}}{\partial v_j} \\
        & = v_j^{(0)} + \frac{\eta}{t} \sum_{p=1}^t \left( \mathbb{E}_{a, b} \left[ \mathbb{E}_{\bm{x} \sim \Pm} \left[ \frac{\partial f(e)_{p, x_{p+1}}}{\partial v_j} \bigg\vert \Pm \right] \right] - \frac{1}{k} \sum_{i = 1}^k \mathbb{E}_{a, b} \left[ \mathbb{E}_{\bm{x} \sim \Pm} \left[ \frac{\partial f(e)_{p, x_{i}}}{\partial v_j} \bigg\vert \Pm \right] \right] \right) \\
        & = c + \frac{c \eta}{t} \sum_{p=j}^t \mathbb{E}_{a, b} \left[ \frac{1}{(\lambda-1)^2} \left( \left( (a-1)^2 + (b-1)^2 \right) \left(p-j+1\right) + 2 (a-1)(b-1)\frac{\lambda^{p-j+1} - 1}{1 - \lambda^{-1}} \right) - \frac{1}{2} \left( p - j + 1 \right)  \right] \\
        & = c + \frac{c \eta}{t} \mathbb{E}_{a, b} \left[ \frac{(t - j + 1)(t-j+2)}{2} \left( \frac{(a-1)^2 + (b-1)^2}{(\lambda - 1)^2} - \frac{1}{2} \right) + 2 \lambda \frac{(a-1)(b-1)}{(\lambda - 1)^3} \left( \frac{\lambda - \lambda^{t-j+2}}{1 - \lambda} - t + j - 1 \right) \right].
    \end{split}
\end{equation}
See Figure \ref{fig:grad_calc_step1} (right) for an empirical estimation for some choice of hyperparameters. 
As we did in the case of the gradient of $W_k$, we focus on the leading order terms and we have:
\begin{equation}
    \begin{split}
        v_j^{(1)} & =  c + \frac{c \eta}{t} \frac{(t - j + 1)(t-j+2)}{2} \mathbb{E}_{a, b} \left[\frac{(a-1)^2 + (b-1)^2}{(\lambda - 1)^2} - \frac{1}{2} \right] + O(1),
    \end{split}
\end{equation}
and, since $\mathbb{E}_{a, b} \left[ \frac{(a-1)^2 + (b-1)^2}{(\lambda - 1)^2}\right] = 2 - \ln4 > \frac{1}{2}$, we get a positive bias with quadratic dependence.

\begin{figure}
    \centering
    \includegraphics[scale=0.5]{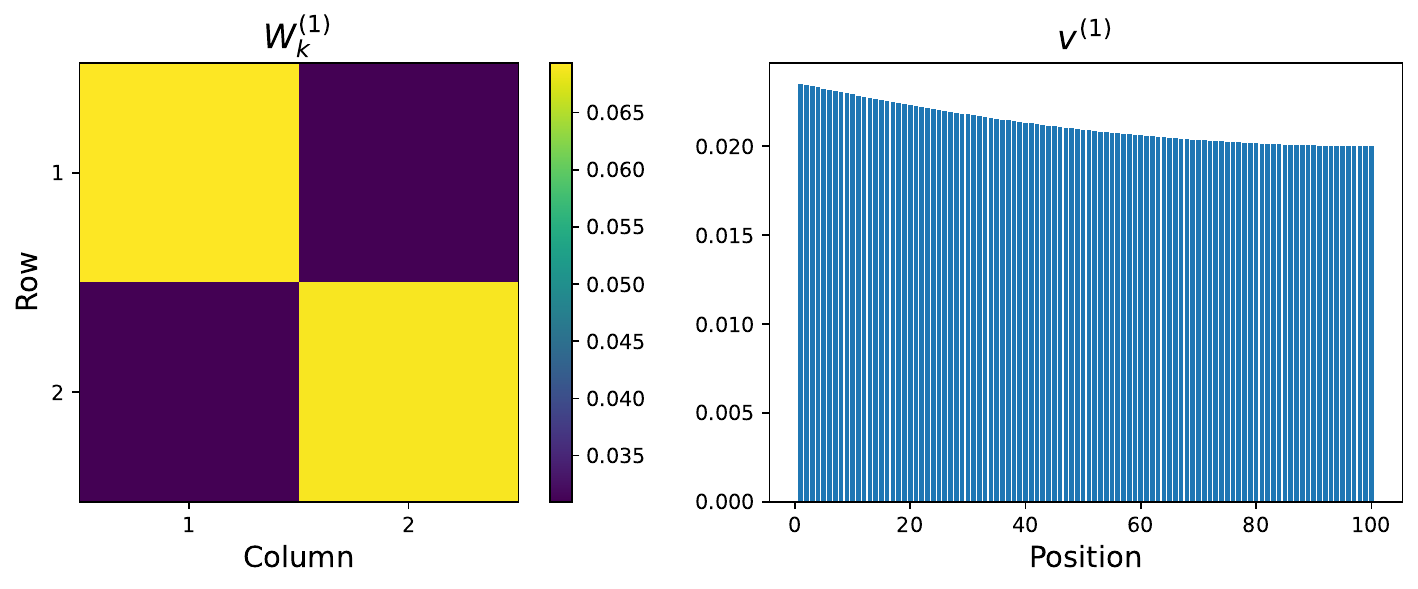}
    \caption{Values of $W_k$ and $v$ after one step of stochastic gradient descent from eqs.~\eqref{eq:Wk_step1},\eqref{eq:v_step1}. Hyperparameters: initialization $c=0.02$, learning rate $\eta=0.03$, sequence length $t=100$. The outer expectations are approximated using 10,000 samples.}
    \label{fig:grad_calc_step1}
\end{figure}

\textbf{Second step.}
We saw that we get a bias towards a diagonal $W_k$ after one update~\eqref{eq:Wk_1st_step}. Let $W_k = C^\prime + (\rho - c^\prime) I$, where $C^\prime = c^\prime 11^T$ and $\rho > c^\prime$, with $c^\prime$ denoting the off-diagonal term and $\rho$ the on-diagonal one. As in the first step, we can tune the margin parameter of the loss, $\Delta$ to be large enough, so that we operate on the linear part of the loss. We have for the gradient of $v$:
\begin{equation}
    \begin{split}
        \mathbb{E} \left[ \frac{\partial f(e)_{p, i}}{\partial v_j} \right] & = \mathbb{E} \left[\sum_{t^\prime = 1}^p \mathds{1}\left\{ x_{t^\prime} = i \right\} \sum_{s=1}^{t^\prime} \delta_{j(t^\prime - s +1)} \left( c^\prime + (\rho - c^\prime) e_{x_p}^T e_{x_s} \right) \mathds{1}\left\{ j \leq p \right\}\right] \\
        & = \mathbb{E} \left[ \sum_{t^\prime=1}^p \mathds{1}\left\{ x_{t^\prime} = i \right\} (c^\prime + \left(\rho - c^\prime\right) e_{x_p}^T e_{x_{t^\prime - j + 1}}) \mathds{1}\left\{ j \leq t^\prime \right\} \mathds{1} \left\{ j \leq p \right\} \right] \\
        & = \mathbb{E} \left[ \sum_{t^\prime=j}^p \mathds{1}\left\{ x_{t^\prime} = i \right\} (c^\prime + \left(\rho - c^\prime\right) \mathds{1} \left\{ x_{t^\prime - j + 1} = x_p \right\} ) \mathds{1} \left\{ j \leq p \right\} \right] \\
        & = c^\prime \pi_i \left( p - j + 1 \right) \mathds{1}\left\{ j \leq p \right\} + \left(\rho - c^\prime\right) \sum_{t^\prime = j}^p \mathbb{P} \left[ x_{t^\prime} = i, x_{t^\prime - j + 1} =  x_p \right] \mathds{1}\left\{ j \leq p \right\}.
    \end{split}
\end{equation}
We split the cases of $j = 1$ and $j \geq 2$.
For $j = 1$:
\begin{equation}      
    \begin{split}
        \mathbb{E} \left[\frac{\partial f(e)_{p, i}}{\partial v_1}\right] & = c^\prime \pi_i  p + \left(\rho - c^\prime\right) \sum_{t^\prime = 1}^p \mathbb{P} \left[ x_{t^\prime} = i =  x_p \right] \\
        & = c^\prime \pi_i  p + \left(\rho - c^\prime\right) \sum_{t^\prime = 1}^p \left(\mathcal{P}^{p-t^\prime}\right)_{ii} \pi_i \\
        & = c^\prime \pi_i  p + \frac{\left(\rho - c^\prime\right) \pi_i}{\lambda - 1} \sum_{t^\prime = 1}^p (\beta_{ii} \lambda^{p-t^\prime} + \gamma_{ii}) \\
        & = c^\prime \pi_i  p + \frac{\left(\rho - c^\prime\right) \pi_i}{\lambda - 1} \left( p \gamma_{ii} + \beta_{ii} \frac{\lambda^p - 1}{\lambda - 1} \right).
    \end{split}        
\end{equation}
\begin{equation}\label{eq:grad_2nd_i}  
    \begin{split}
        \mathbb{E} \left[\frac{\partial f(e)_{p, i}}{\partial v_1}\right] & = c^\prime \pi_i \left( p - j + 1 \right) \mathds{1}\left\{ j \leq p \right\} + \left(\rho - c^\prime\right) \sum_{t^\prime = j}^p \sum_{l = 1}^2 \mathbb{P} \left[ x_{t^\prime} = i, x_{t^\prime - j + 1} =  x_p = l \right] \mathds{1}\left\{ j \leq p \right\} \\
        & =  c^\prime \pi_i \left( p - j + 1 \right) \mathds{1}\left\{ j \leq p \right\} + \left(\rho - c^\prime\right) \sum_{l = 1}^2 \sum_{t^\prime = j}^p \mathbb{P} \left[ x_p = l \mid x_{t^\prime} = i \right] \mathbb{P} \left[ x_{t^\prime} = i \mid x_{t^\prime - j + 1} = l \right] \mathbb{P} \left[ x_{t^\prime - j + 1} = l \right] \mathds{1} \left\{ j \leq p \right\} \\
        & =  c^\prime \pi_i \left( p - j + 1 \right) \mathds{1}\left\{ j \leq p \right\} + \left(\rho - c^\prime\right) \sum_{l = 1}^2 \pi_l \sum_{t^\prime = j}^p \frac{ \beta_{il} \lambda^{p - t^\prime} + \gamma_{il}}{\lambda - 1} \frac{ \beta_{li} \lambda^{j-1} + \gamma_{li}}{\lambda - 1} \mathds{1} \left\{ j \leq p \right\} \\
        & = \left( c^\prime \pi_i \left( p - j + 1 \right) + \left(\rho - c^\prime\right) \sum_{l=1}^2 \frac{\pi_l}{(\lambda-1)^2} \sum_{t^\prime=j}^p \left( \beta_{il} \beta_{li} \lambda^{p-t^\prime+j-1} + \beta_{il} \gamma_{li} \lambda^{p-t^\prime} + \gamma_{il} \beta_{li} \lambda^{j-1} + \gamma_{il} \gamma_{li} \right) \right) \mathds{1} \left\{ j \leq p \right\} \\
        & = c^\prime \pi_i \left( p - j + 1 \right) \mathds{1}\left\{ j \leq p \right\} \\ & \; \; \; + \frac{\left(\rho - c^\prime\right)}{(\lambda-1)^2} \sum_{l=1}^2 \pi_l \left( \beta_{il} \beta_{li} \frac{\lambda^{p} - \lambda^{j-1}}{\lambda - 1} + \beta_{il} \gamma_{li} \frac{\lambda^{p-j+1} - 1}{\lambda - 1} + (p-j+1)\left( \gamma_{il} \beta_{li} \lambda^{j-1} + \gamma_{il} \gamma_{li} \right) \right) \mathds{1} \left\{ j \leq p \right\}.
    \end{split}
\end{equation}
and, similarly, for the derivative of the ground truth logit:
\begin{equation}
    \begin{split}
        \mathbb{E} \left[ \frac{\partial f(e)_{p, x_{p+1}}}{\partial v_j} \right] & = c^\prime \sum_{t^\prime = j}^p \mathbb{P} \left[ x_{t^\prime} = x_{p+1} \right] \mathds{1}\left\{ j \leq p \right\} + \left(\rho - c^\prime\right) \sum_{t^\prime = j}^p \mathbb{P} \left[ x_{t^\prime} = x_{p+1}, x_{t^\prime - j + 1} =  x_p \right] \mathds{1}\left\{ j \leq p \right\}.
    \end{split}
\end{equation}
The first term corresponds to the gradient of the first step and can be found in eq.~\eqref{eq:v1_grad_ground}. For the second term, we have for $j = 1$:
\begin{equation}      
    \begin{split}
        \sum_{t^\prime = 1}^p \mathbb{P} \left[ x_{t^\prime} = x_{p+1} =  x_p \right] & = \sum_{i = 1}^2 \sum_{t^\prime = 1}^p \mathbb{P} \left[ x_{p+1} = i \mid x_p = i \right] \mathbb{P} \left[ x_{p} = i \mid x_{t^\prime} = i \right] \mathbb{P} \left[ x_{t^\prime} = i \right] \\
        & = \sum_{i = 1}^2 \frac{\pi_i \mathcal{P}_{ii}}{\lambda - 1} \left( \beta_{ii} \frac{\lambda^p - 1}{\lambda - 1} + \gamma_{ii} p \right),
    \end{split}        
\end{equation}
and for $ j \geq 2$:
\begin{equation}      
    \begin{split}
        \sum_{t^\prime = j}^p \mathbb{P} \left[ x_{t^\prime} = x_{p+1}, x_{t^\prime - j + 1} =  x_p \right] & = \sum_{i = 1}^2 \sum_{l = 1}^2 \sum_{t^\prime = j}^p \mathcal{P}_{li} \mathbb{P} \left[ x_p = l \mid x_{t^\prime} = i \right] \mathbb{P} \left[ x_{t^\prime} = i \mid x_{t^\prime - j + 1} = l \right] \mathbb{P} \left[ x_{t^\prime - j + 1} = l \right],
    \end{split}
\end{equation}
which, by using the calculations of eq.~\eqref{eq:grad_2nd_i}, amounts to:
\begin{equation}
    \frac{1}{(\lambda-1)^2} \sum_{i=1}^2 \sum_{l=1}^2 \pi_l \mathcal{P}_{li} \left( \beta_{il} \beta_{li} \frac{\lambda^{p} - \lambda^{j-1}}{\lambda - 1} + \beta_{il} \gamma_{li} \frac{\lambda^{p-j+1} - 1}{\lambda - 1} + (p-j+1)\left( \gamma_{il} \beta_{li} \lambda^{j-1} + \gamma_{il} \gamma_{li} \right) \right) \mathds{1} \left\{ j \leq p \right\}.
\end{equation}

Thus, the updates will have a gradient contribution that is the same as the first step, $\frac{\partial \mathcal{L}^0}{\partial v_j}$, and another one that comes from the diagonal bias of $W_k$.

For $j=1$:
\begin{equation}\label{eq:v_step2_1}
    \begin{split}
         v_1^{(2)} & = v_1^{(1)} - \eta \frac{\partial \mathcal{L}}{\partial v_1} \\
    & = v_1^{(1)} - \eta c^\prime \frac{\partial \mathcal{L}^0}{\partial v_1} + \frac{\eta \left(\rho - c^\prime\right)}{t} \sum_{p=1}^t \sum_{i = 1} ^2 \mathbb{E}_{a, b} \left[ \left( \mathcal{P}_{ii} - \frac{1}{2} \right) \frac{\pi_i}{\lambda - 1} \left( \beta_{ii} \frac{\lambda^p - 1}{\lambda - 1} + \gamma_{ii} p \right) \right]\\
    & = v_1^{(1)} - \eta c^\prime \frac{\partial \mathcal{L}^0}{\partial v_1} + \frac{\eta \left(\rho - c^\prime\right)}{t} \mathbb{E}_{a, b} \left[\left( \frac{\lambda - \lambda^{t+1}}{1 - \lambda} - t \right) \frac{(a-1)(b-1)}{\left(\lambda - 1 \right)^2} + \left( \frac{(b - 1)^2 (a - \frac{1}{2}) + (a-1)^2 (b - \frac{1}{2})}{(\lambda - 1)^2} \right) \frac{t(t+1)}{2} \right].
    \end{split}
\end{equation}
For $j \geq 2$:
\begin{equation}\label{eq:v_step2_2}
    \begin{split}
         v_j^{(2)} & = v_j^{(1)} - \eta \frac{\partial \mathcal{L}}{\partial v_j} \\
    & = v_1^{(1)} - \eta c^\prime \frac{\partial \mathcal{L}^0}{\partial v_j} + \frac{\eta \left(\rho - c^\prime\right)}{t} \sum_{i = 1} ^2 \sum_{l=1}^2 \mathbb{E}_{a, b} \Bigg[ \frac{\pi_l}{(\lambda - 1)^2} \left( \mathcal{P}_{li} - \frac{1}{2} \right) \bigg[ \frac{\beta_{il} \beta_{li}}{\lambda - 1} \left( \frac{\lambda^j - \lambda^{t+1}}{1 - \lambda} - (t-j+1)\lambda^{j-1} \right) \\
    & \;\;\;\;\;\;\;\;\;\;\;\;\;\;\;\;\;\;\;\;\;\;\;\;\;\;\;\;\;\;\;\;\;\;\;\;\;\;\;\;\;\;\;\;\;\;\;\;\;\;\;\;\;\;\;\;\;\;\;\;\;\;\;\;\;\;\;\;\;\;\;\;\;\;\;\;\;\;\;\;\;\;\;\;\;\;\;\ + \frac{\beta_{il} \gamma_{li}}{\lambda - 1} \left( \frac{\lambda - \lambda^{t-j+2}}{\lambda - 1} - t + j -1 \right) \\
    & \;\;\;\;\;\;\;\;\;\;\;\;\;\;\;\;\;\;\;\;\;\;\;\;\;\;\;\;\;\;\;\;\;\;\;\;\;\;\;\;\;\;\;\;\;\;\;\;\;\;\;\;\;\;\;\;\;\;\;\;\;\;\;\;\;\;\;\;\;\;\;\;\;\;\;\;\;\;\;\;\;\;\;\;\;\;\;\ + \frac{(t-j+1)(t-j+2)}{2} \left( \gamma_{il} \beta_{li} \lambda^{j-1} + \gamma_{il} \gamma_{li} \right) \bigg] \Bigg].
    \end{split}
\end{equation}
See Figure \ref{fig:grad_calc_step2} for an empirical estimation of the diagonal contribution of the gradient for some choice of hyperparameters.  In leading order terms, the updates simplify to:

\begin{figure}
    \centering
    \includegraphics[scale=0.5]{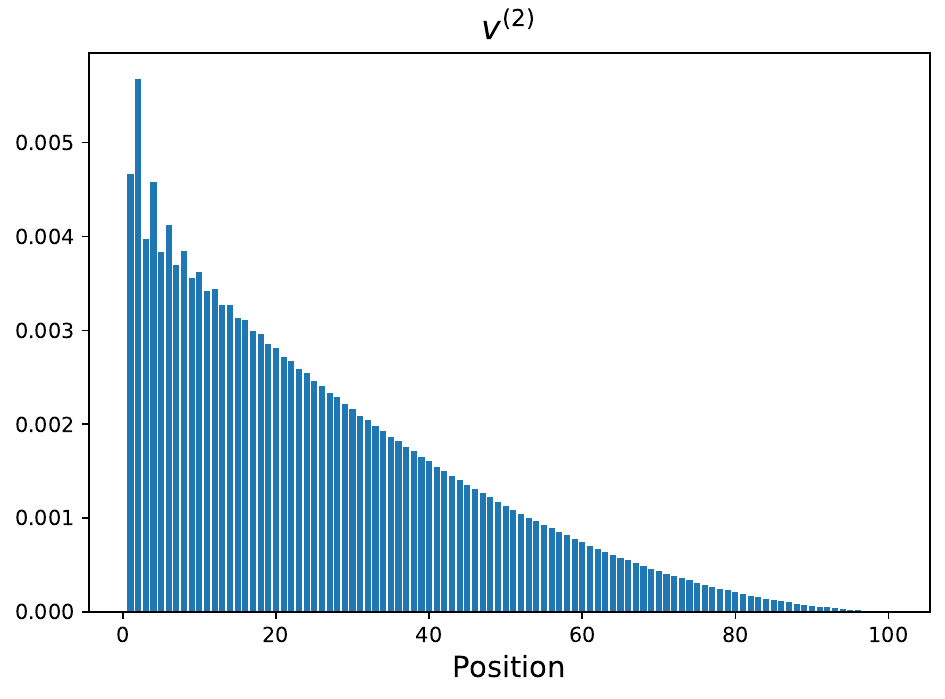}
    \caption{Empirical estimation of the 2nd step gradient of $v$ that comes from the diagonal bias of $W_k$ - see eqs.~\eqref{eq:v_step2_1},\eqref{eq:v_step2_2}. Hyperparameters: initialization $c=0.02$, learning rate $\eta=0.03$, sequence length $t=100$. The outer expectations are approximated using 10,000 samples.}
    \label{fig:grad_calc_step2}
\end{figure}

For $j = 1$:
\begin{equation}\label{eq:2ndstep_1}
    v^{(2)}_1 = v_1^{(1)} - \eta c^\prime \frac{\partial \mathcal{L}^{0}}{\partial v_1} + \frac{\eta \left(\rho - c^\prime\right)}{t} \mathbb{E}_{a, b} \left[ \frac{(b - 1)^2 (a - \frac{1}{2}) + (a-1)^2 (b - \frac{1}{2})}{(\lambda - 1)^2} \right] \frac{t(t+1)}{2} + O(1).
\end{equation}
For $j \geq 2$:
\begin{equation}\label{eq:2ndstep_j}
    \begin{split}
        v_j^{(2)} = v_j^{(1)} - \eta c^\prime \frac{\partial \mathcal{L}^0}{\partial v_j} + \frac{\eta \left(\rho - c^\prime\right)}{t} \frac{(t-j+1) (t-j+2)}{2} \mathbb{E}_{a, b} & \bigg[ 2 \frac{(a-1)(b-1)\left( (b-1) \left(a - \frac{1}{2}\right) + (a-1) \left(b - \frac{1}{2}\right) \right)}{(\lambda-1)^3} \lambda^{j-1}  \\ 
        & + \frac{(b-a)\left( (b-1)^2 \left(a - \frac{1}{2}\right) - (a-1)^2 \left(b - \frac{1}{2}\right) \right)}{(\lambda-1)^3} \bigg] + O(1),
    \end{split}
\end{equation}
where for all $j$ it is:
\begin{equation}\label{eq:init-one}
    v_j^{(1)} - \eta c^\prime \frac{\partial \mathcal{L}^{0}}{\partial v_j} = c + \frac{(c + c^\prime) \eta}{t} \frac{(t-j+1)(t-j+2)}{2} \mathbb{E}_{a, b} \left[ \frac{(a-1)^2 + (b-1)^2}{(\lambda - 1)^2} - \frac{1}{2} \right] + O(1).
\end{equation}

We now show that $v^{(2)}_2 > v^{(2)}_1$ in the large $t$ regime. Observe that the term of \eqref{eq:init-one} is the same for $j=1,2$. We evaluate the expectation of eq. \eqref{eq:2ndstep_1} that corresponds to $j = 1$:
\begin{equation}
    \mathbb{E}_{a, b} \left[ \frac{(b - 1)^2 (a - \frac{1}{2}) + (a-1)^2 (b - \frac{1}{2})}{(\lambda - 1)^2} \right] = \frac{1 - \ln2}{3}
\end{equation}
and the expectation of \eqref{eq:2ndstep_j}, leveraging symmetry of $a, b$ in the expression, that corresponds to $j = 2$:
\begin{equation}
    4 \mathbb{E}_{a, b} \left[ \frac{(a-1)(b-1)^2(a-\frac{1}{2})}{(a+b-2)^3} (a+b-1)\right] + 2 \mathbb{E}_{a, b} \left[ \frac{(b-a)(b-1)^2(a-\frac{1}{2})}{(a+b-2)^3} \right] = \frac{7 - 10 \ln 2}{2} + \frac{5 - 8\ln2}{6}.
\end{equation}
The latter term is larger, so, when $t \to \infty$, it is $v^{(2)}_2 > v^{(2)}_1$. Finally, observe that the expectation $\mathbb{E}_{a, b} \left[ \frac{(a-1)(b-1)^2(a-\frac{1}{2})}{(a+b-2)^3} (a+b-1)^{j-1}\right]$ is negative for $j$ odd and decreasing otherwise, which shows that $v^{(2)}_2 > v^{(2)}_j, j > 2$. So, we showed that the 2nd coordinate grows larger than the rest after the 2nd step.

\end{proof}

With this calculation at hand, we can ask questions about the optimization trajectory and how that would change if we alter the data distribution. In particular, in the following Corollary, we answer the question of how the gradient would change when $a=b$, i.e., when the transition matrix is
\begin{equation}
    \Pm = \begin{pmatrix}    
        a & 1-a \\ 1-a & a
    \end{pmatrix},
\end{equation}
with $a \sim Unif([0, 1])$. Notice that in the case there is not a unigrams solution (or in other words the unigrams solution is as good as the uniformly at random solution).
\begin{corollary}\label{cor:aequalb}
When $a=b$, it is 
\begin{equation}
    \begin{split}
        W_k^{(1)} & = \begin{pmatrix}c & c \\ c & c\end{pmatrix} + c \eta \left[ O(t) \begin{pmatrix}\infty & 0 \\ 0 & \infty\end{pmatrix}  + O(1) \right], \\
        v_j^{(1)} & = c - \frac{c \eta}{2t} \left(t-j+1\right) + O(1/t), j \in [t].
    \end{split} 
\end{equation}
\end{corollary}
\begin{proof}
    We first compute the gradient with respect to $W_k$, by setting $a=b$ in eq.~\eqref{eq:Wk_11},\eqref{eq:Wk_12}. For $m, l = 1, 1$, we have:
    \begin{equation}
        \left(W_{k}^{(1)}\right)_{1, 1} = c + \frac{\eta c}{t} \mathbb{E}_{a \sim Unif([0, 1])} \left[ \frac{(a-1)(a-\frac{1}{2})}{(2a-2)^3} \left[\frac{2(a-1)^2}{(2a-2)^2} (II) + \frac{2(a-1)^2}{2a-2} (III) \right] \right],
    \end{equation}
    where $(II), (III)$ are defined in eq.~\eqref{eq:Wk_1st_step}. Notice that the $(IV)$ with the cubic, $O(t^3)$, dependence disappeared, so the leading order terms are $O(t^2)$. By substituting the leading order terms from $(II), (III)$, we get:
    \begin{equation}\label{eq:aequalb_for_proof}
        \begin{split}
            \left(W_{k}^{(1)}\right)_{1, 1} & = c + \frac{\eta c}{t} \mathbb{E}_{a \sim Unif([0, 1])} \left[ \frac{(a-1)(a-\frac{1}{2})}{(2a-2)^3} \left[\frac{2(a-1)^2}{(2a-2)^2} \frac{2 - 2a}{2}  - \frac{2(a-1)^2}{2a-2} \frac{1}{2} \right] \right] O(t^2) \\
            & = c + \eta c \mathbb{E}_{a \sim Unif([0, 1])} \left[ - \frac{(a-\frac{1}{2})}{8 (a-1)} \right] O(t) \to \infty.        
        \end{split}
    \end{equation}
    For $m, l = 1, 2$, the $O(t^3)$ terms similarly disappear, and we have:
    \begin{equation}
        \begin{split}
            \left(W_{k}^{(1)}\right)_{1, 2} & = c + \frac{\eta c}{t} \mathbb{E}_{a \sim Unif([0, 1])} \left[ \frac{(a-1)(a-\frac{1}{2})}{(2a-2)^3} \left[\frac{2(a-1)^2}{(2a-2)^2} (II) - \frac{2(a-1)^2}{2a - 2} (III) \right] \right] \\
            & = c + \frac{\eta c}{t} \mathbb{E}_{a \sim Unif([0, 1])} \left[ \frac{(a-1)(a-\frac{1}{2})}{(2a-2)^3} \left[\frac{2(a-1)^2}{(2a-2)^2} \frac{2 - 2a}{2}  + \frac{2(a-1)^2}{2a-2} \frac{1}{2} \right] \right] O(t^2) \\
            & = c.
        \end{split}
    \end{equation}
    Now, for the $v$ vector, we observe that when we set $a=b$ in eq.~\eqref{eq:v_step1}, the leading term disappears and its expression simplifies to:
    \begin{equation}\label{eq:aequalb_proof_v}
        \begin{split}
             v_j^{(1)} & = c - \frac{c \eta}{t} \mathbb{E}_{a \sim Unif([0, 1])} \left[2 (2a-2) \frac{(a-1)^2}{(2a -2)^3} \left(t - j + 1 \right) \right] + O(1/t) \\
             & = c - \frac{c \eta}{2t} \left(t-j+1\right) + O(1/t).
        \end{split}
    \end{equation}
\end{proof}
The previous calculations reveal that there is less ``signal" towards the unigrams solution when $a=b$.

Finally, we provide our main technical result that demonstrates that two steps of stochastic gradient descent lead the model to the bigrams solution.

\begin{proposition}
Suppose we initialize the model of eq.~\eqref{eq:model} with $W_k^{(0)} = w_{init} 11^T$ and $v = v_{init} 1^T$. Consider a two step procedure, where we first train with gradient descent on the original distribution $a \sim \mathrm{Unif} (0, 1),  b \sim \mathrm{Unif} (0, 1)$, and then at the second step with $a = b \sim \mathrm{Unif} ([\frac{1}{2} - \epsilon, \frac{1}{2} + \epsilon])$, $\epsilon \in (0, 1/2)$. Let $\eta_{1} = O\left(\frac{1}{t^2}\right)$ and $\eta_{2} = O\left(\frac{1}{t}\right)$ be the learning rates for the 1st and 2nd step of gradient descent, and let $\lambda_{wd} = 1$ be the weight decay coefficient used in the second step. Then, after 2 steps of gradient descent, for $t \to \infty$, it is:
\begin{equation}
    f(e)_{p, i} = O(\epsilon^4) \sum_{t^\prime = 1}^p \mathds{1} \left\{x_{t^\prime} = i\right\} \mathds{1} \left\{x_{p} = x_{t^\prime - 1}\right\} + O(\epsilon^6).
\end{equation}
\end{proposition}
\begin{proof}
For the first step, we have from Lemma 1:
\begin{equation}
    W_k^{(1)} = w_{init}11^T + v_{init} \eta_1 \begin{pmatrix}B & A \\ A & B\end{pmatrix} O(t^2) + v_{init} \eta_1 O(t).
\end{equation}
For $\eta_1 = O\left(\frac{1}{t^2}\right)$, then
\begin{equation}
    W_k^{(1)} = w_{init}11^T + \begin{pmatrix}B & A \\ A & B \end{pmatrix} O(1) + O(1/t).
\end{equation}
Recall, from Lemma 1, that we have $v^{(1)}_j = v_{init} + \eta_1 w_{init} O(t)$ for all $j \in [t]$, so, for $\eta_1 = O\left(\frac{1}{t^2}\right)$, it will be $v_j^{(1)} = v_{init}$ for all $j \in [t]$.

Assume now that at the second step we train with weight decay $\lambda_{wd} = 1$ and the data distribution changes to $a=b \sim \mathrm{Unif} ([\frac{1}{2} - \epsilon, \frac{1}{2} + \epsilon])$. Then, from eq.~\eqref{eq:aequalb_for_proof}\footnote{That equation refers to the first step of gradient descent for $a=b$, but notice that since the model is linear in $W_k$, and in this case $v_j^{(1)} = v_{init}$, it characterizes the 2nd step too.} the update of the 2nd layer would be:
\begin{equation}
    \begin{split}
        W_k^{(2)} & = (1 - \lambda_{wd}) W_k^{(1)} - \eta_2 \nabla \mathcal{L} \\
        & = v_{init} \eta_2 \begin{pmatrix} \int_{\frac{1}{2} - \epsilon}^{\frac{1}{2} + \epsilon} \frac{a - \frac{1}{2}}{(a - 1)} & 0 \\ 0 & \int_{\frac{1}{2} - \epsilon}^{\frac{1}{2} + \epsilon} \frac{a - \frac{1}{2}}{(a - 1)} \end{pmatrix} O(t) + O(1) \\
        & = v_{init} \eta_2 \begin{pmatrix} 2\epsilon + \frac{1}{2} \ln \left( \frac{\frac{1}{2} - \epsilon}{\frac{1}{2} + \epsilon} \right)  & 0 \\ 0 & 2\epsilon + \frac{1}{2} \ln \left( \frac{\frac{1}{2} - \epsilon}{\frac{1}{2} + \epsilon} \right) \end{pmatrix} O(t) + O(1).
    \end{split}
\end{equation}
For the update of $v$ we have:
\begin{equation}
    \begin{split}
        v_j^{(2)} & = (1 - \lambda_{wd}) v_j^{(1)} - \frac{\partial \mathcal{L}}{\partial v_j} \\
        & = \eta_2 A \frac{\partial \mathcal{L}_{const}}{\partial v_j} + \eta_2 3A \frac{\partial \mathcal{L}_{diag}}{\partial v_j} \\
        & = A \frac{\eta_2 \rho}{t} \left( \frac{(t - j + 1)(t-j+2)}{2} \left( \frac{(a-1)^2 + (b-1)^2}{(\lambda - 1)^2} - \frac{1}{2} \right) + O(t) \right) +  3A \eta_2 \frac{\partial \mathcal{L}_{diag}}{\partial v_j}.
    \end{split}
\end{equation}
For $a = b$, it is only the $O(t)$ part inside the parenthesis that survives from the second term (see eq.~\eqref{eq:aequalb_proof_v}) and by substituting the diagonal gradient from eqs.~\eqref{eq:2ndstep_1},\eqref{eq:2ndstep_j} when $a=b$, we get for all $j \in [t]$:
\begin{equation}
    \begin{split}
        v_j & = \eta_2 \rho A O(1) + 3A \frac{\eta_2 }{t} \frac{(t-j+1)(t-j+2)}{2} 2 \mathbb{E}_{a \sim Unif([1/2 - \epsilon, 1/2 + \epsilon])} \left[ \frac{a - 1 / 2}{4} (2a - 1)^{j-1} \right] + 3A \frac{\eta_2 }{t} O(t).
    \end{split}
\end{equation}
We calculate the expectation:
\begin{equation}
    \mathbb{E}_{a \sim Unif([1/2 - \epsilon, 1/2 + \epsilon])} \left[ \frac{a - 1 / 2}{4} (2a - 1)^{j-1} \right] = 0,
\end{equation}
when $j = 1$. For $j \geq 2$:
\begin{equation}
    \begin{split}
        \mathbb{E}_{a \sim Unif([1/2 - \epsilon, 1/2 + \epsilon])} \left[ (a - 1 / 2) (2a - 1)^{j-1} \right] & = \int_{1/2 - \epsilon}^{1/2 + \epsilon} (a - 1 / 2) (2a - 1)^{j-1} \\
        & = \frac{1}{2j} \left[ (2a-1)^j (a-\frac{1}{2}) \right]_{1/2-\epsilon}^{1/2+\epsilon} - \frac{1}{2j} \int_{1/2 - \epsilon}^{1/2 + \epsilon} (2a-1)^j da \\
        & = \frac{1}{2j} \epsilon \left( (2\epsilon)^j + (-2\epsilon)^j \right) - \frac{1}{4j(j+1)} \left( (2\epsilon)^{j+1} - (-2\epsilon)^{j+1} \right).
    \end{split}
\end{equation}
For $j = 2k+1$, this equals to 0. For $j = 2k$, it is equal to $\frac{4^k}{2k+1}\epsilon^{2k+1}$.

Thus, by setting $\eta_2 = O\left( \frac{1}{t} \right)$, we get:
\begin{equation}
    \begin{split}
        W_k^{(2)} & = O(\epsilon) \begin{pmatrix} 1  & 0 \\ 0 & 1\end{pmatrix} + O\left( \frac{1}{t} \right).
    \end{split}
\end{equation}
\begin{equation}
    v_j = v_{init} + O\left(\frac{(t-j+1)(t-j+2)}{t^2}\right) \sum_{s = 1, \\ s=2k}^t \delta(s-j) \epsilon^{j+1} + O\left(\frac{1}{t}\right),
\end{equation}
where $\delta$ is Dirac's delta function.

Thus, after 2 steps of gradient descent, as $t \to \infty$, the prediction of the model will be
\begin{equation}
    f(e)_{p, i} = O(\epsilon^4) \sum_{t^\prime = 1}^p \mathds{1} \left\{x_{t^\prime} = i\right\} \mathds{1} \left\{x_{p} = x_{t^\prime - 1}\right\} + O(\epsilon^6).
\end{equation} 
\end{proof}

An interesting observation that stems from the proof is that the learning rate of the first step is $O(1/t^2)$, where recall $t$ is the sequence length, while for the second step the learning rate is $O(1/t)$ (much larger). The first step in our proof corresponds to the learning of the second layer $W_k$, while the second step ``cleans up" $W_k$ and learns the first layer $v$. Interestingly, this is what we also observe in the experiments with the transformers and this minimal model: the first drop is immediate for the learning of the 2nd layer, while the second loss drop happens after a long plateau and corresponds to the ``grokking" of the positional embeddings.
\section{Experimental Details}
We train transformers of the form \eqref{eq:tf-def} with the AdamW optimizer with learning rate $3e-5$, batch size $64$, and hidden dimension $16$. The sequence length of the examples is $100$ tokens. The minimal model was trained with SGD, with batch size $64$, and learning rate $2e-3$. For $3$-grams, a learning rate of $3e-2$ was used. We use PyTorch 2.1.2. Some of the training and model code was based on minGPT \cite{MinGPT}.


The experiments all measure the outputs of the models at the last token.
\begin{figure}
    \centering
    \includegraphics[width=0.6\linewidth]{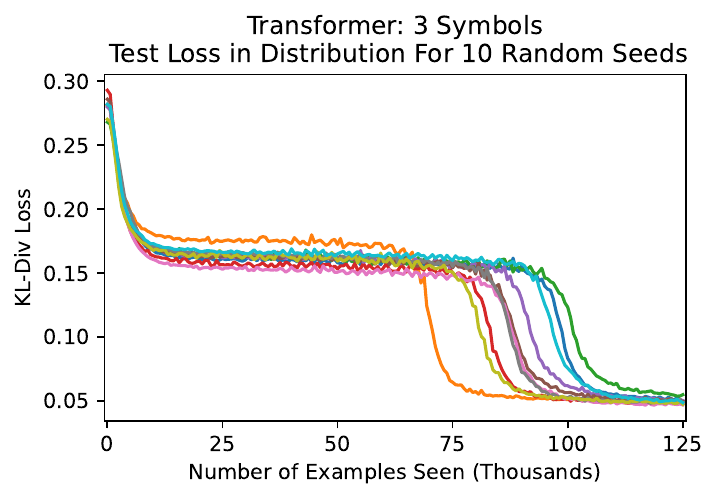}
    \caption{In distribution test loss for $10$ two layer attention only transformers, with random seeds $0,1,\dots 9$ (randomness affects initialization and the training data). The training dynamics are consistent for each model, though the exact position of the phase transitions.}
\end{figure}
\begin{figure}
    \centering
    \includegraphics[width=0.6\linewidth]{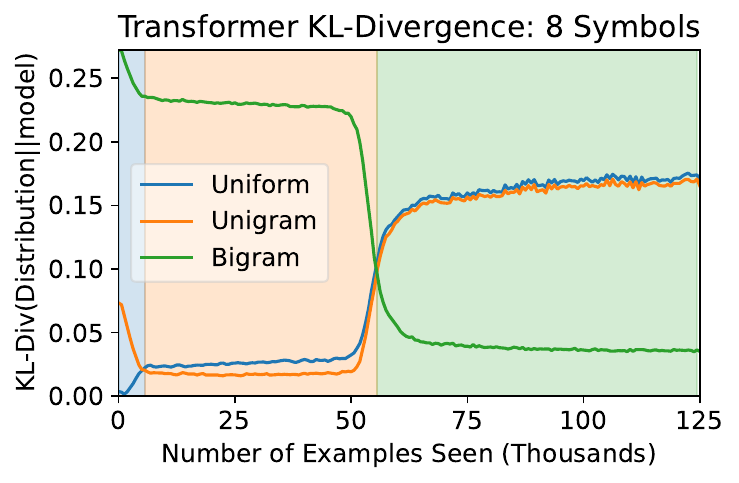}
    \caption{Our results extend to more symbols than $k=2$ or $k=3$. The KL-divergence between the transformer and strategies over training. This required a sequence length greater than $100$ ($200$ in this case) for the difference between unigrams and bigrams to be large enough for the unigram phase to be visible (in either case there was a plateau before the final drop in test loss).}
\end{figure}
\paragraph{Distributions Used in Figure \ref{fig:interpolate}}
\label{para:interpolate}
This graph is for $k=2$ symbols, to simplify the distributions used. The exact method of generating the distributions used in Figure \ref{fig:interpolate} is as follows. Let $p\in (0,1)$ parametrize the distribution. Choose a uniformly random number $x\sim U(0,1)$, then define $\mu = x + p (1- 2x)$. Then choose $y\sim U(\mu- 0.2, \mu+0.2)$, and then $y = max(0,min(1,y))$. Then, with probability $1/2$, $a=x$ and $b=y$, otherwise $a=y$ and $b=x$. Then the transition matrix is
$$\begin{pmatrix}
    a & 1-a \\
    b & 1-b
\end{pmatrix}.$$

\section{Additional Experiments}
\begin{figure}[b!]
    \includegraphics[width=\textwidth]{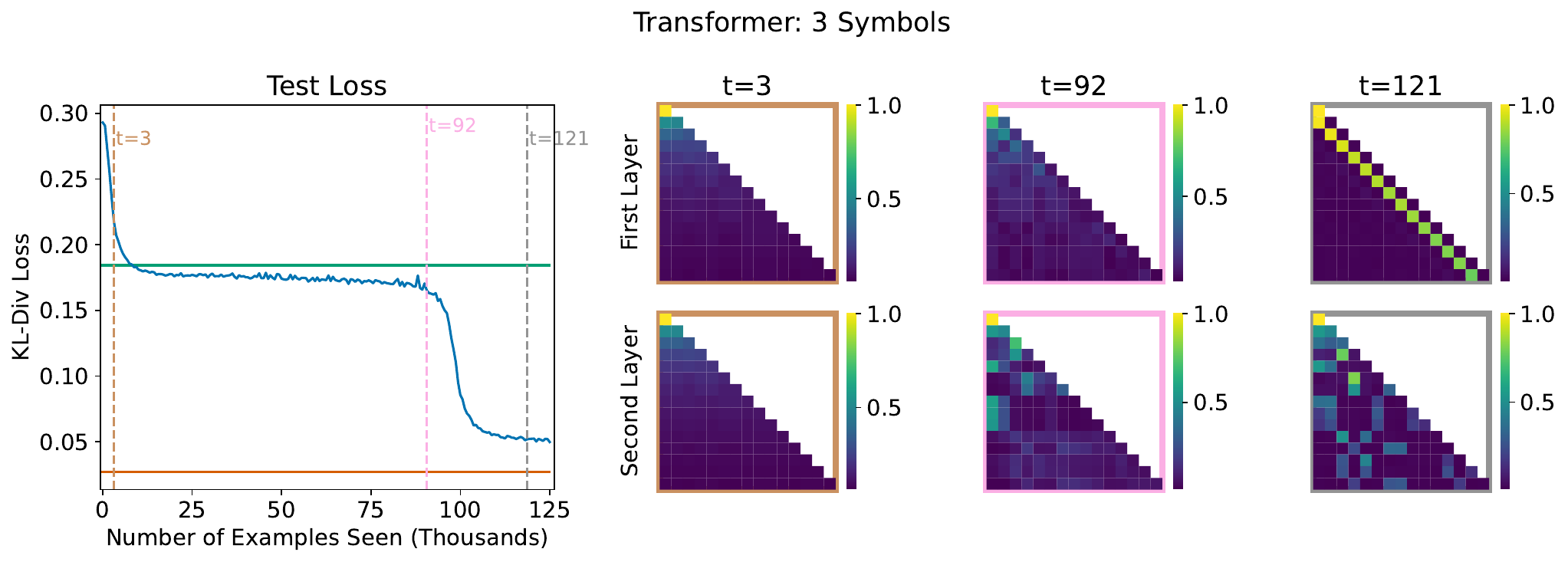}
    \caption{A two layer attention-only transformer was trained with cross entropy loss on ICL-MC. The heatmaps on the right represent part of the attention for the transformer at various time steps, specifically the value of $A$ from \eqref{eq:attn-def}. The top row are showing $A$ from the first layer, and the bottom row from the second layer.}
    \label{fig:attn_heatmap}
\end{figure}

\begin{figure*}[t!]
    \centering
    \includegraphics[width=0.46\linewidth]{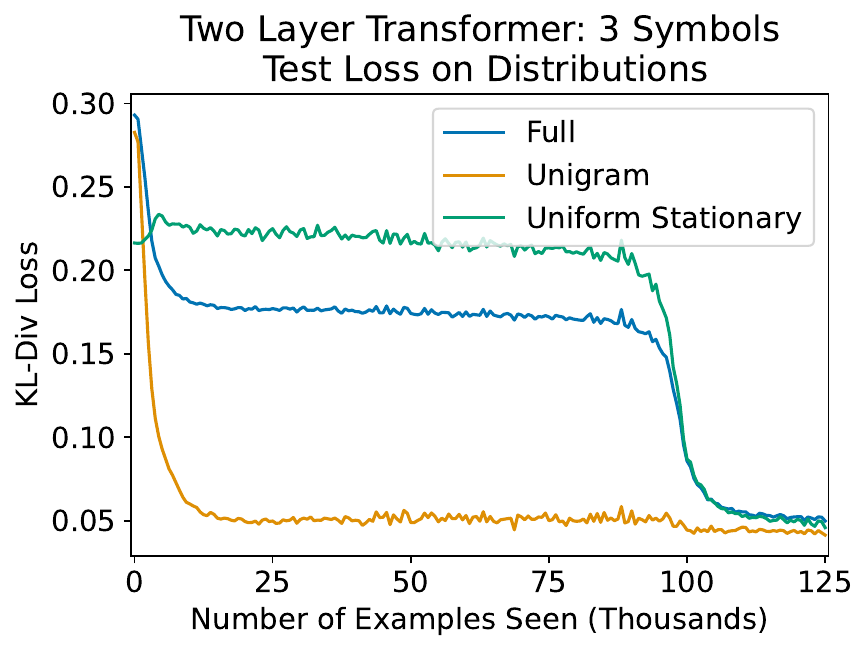}
    \includegraphics[width=0.46\linewidth]{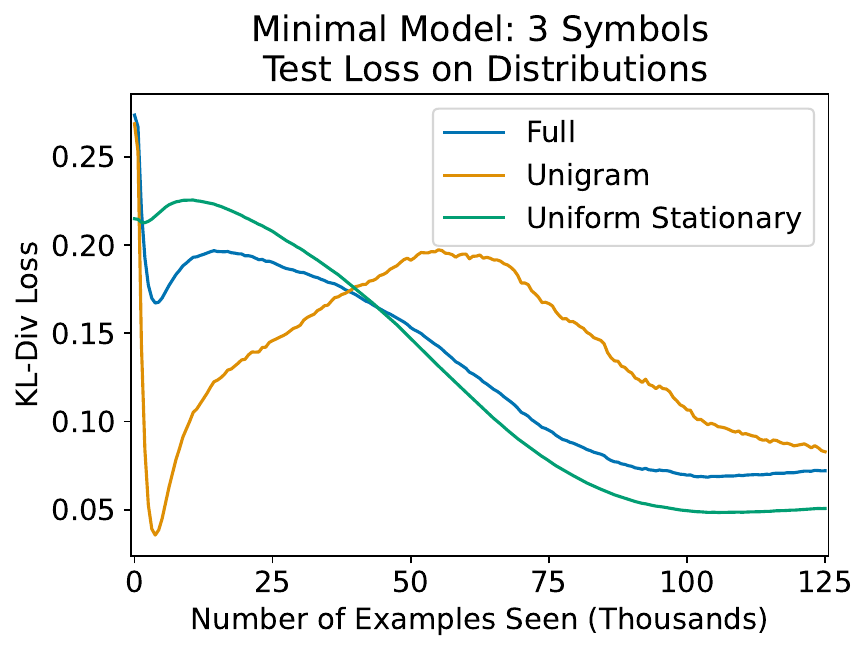}
    \caption{A two layer attention-only transformer (top) and minimal model (\ref{eq:min-mod-def}) (bottom), trained on the main task with ICL-MC with cross entropy loss, test loss measured by KL-Divergence from the underlying truth (labels based on transition probabilities, not samples). The distributions test loss is measured in are (from left to right) in-distribution, a distribution where each token is sampled iid, and a distribution over uniformly random doubly stochastic transition matrices (equivalently, stationary distribution is identity, or unigram based guesses are as good as guessing uniform probability). 
    For both models, the in distribution test loss quickly drops to the level of the unigram algorithm. }
    \label{fig:tf_loss}
\end{figure*}
\begin{figure*}
    \includegraphics[width=0.65\textwidth]{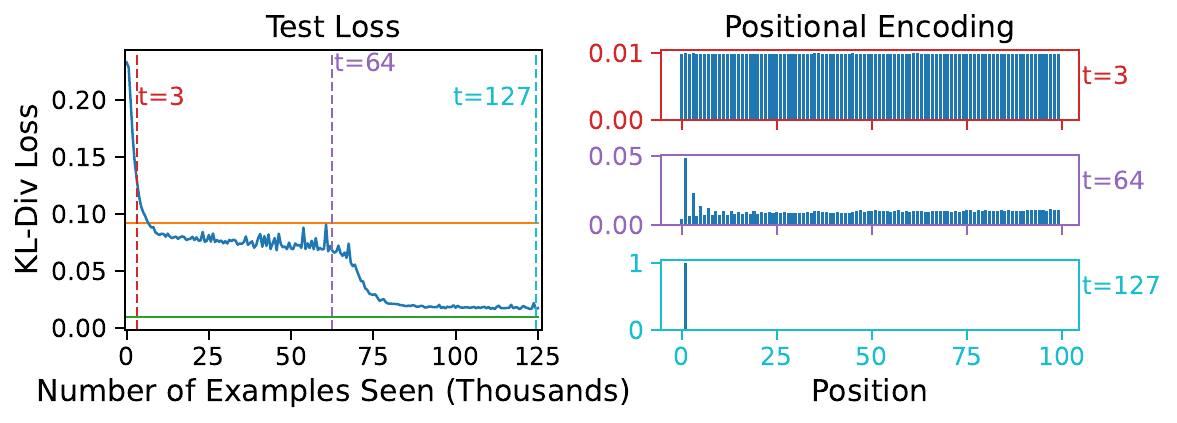}
    \includegraphics[width=0.36\textwidth]{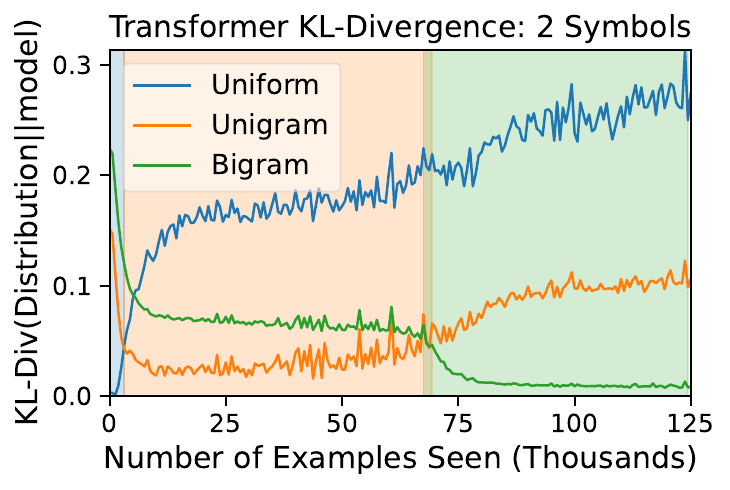}
    \includegraphics[width=0.65\textwidth]{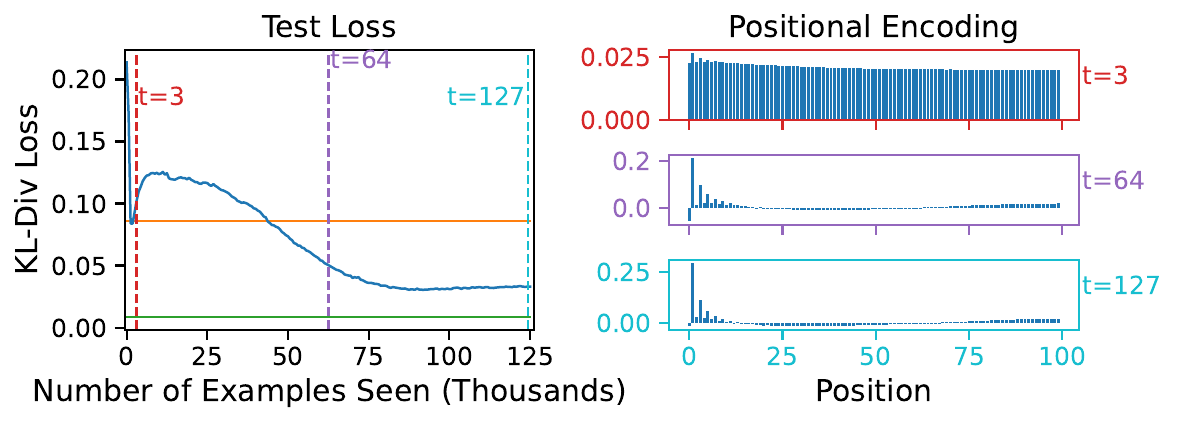}
    \includegraphics[width=0.36\textwidth]{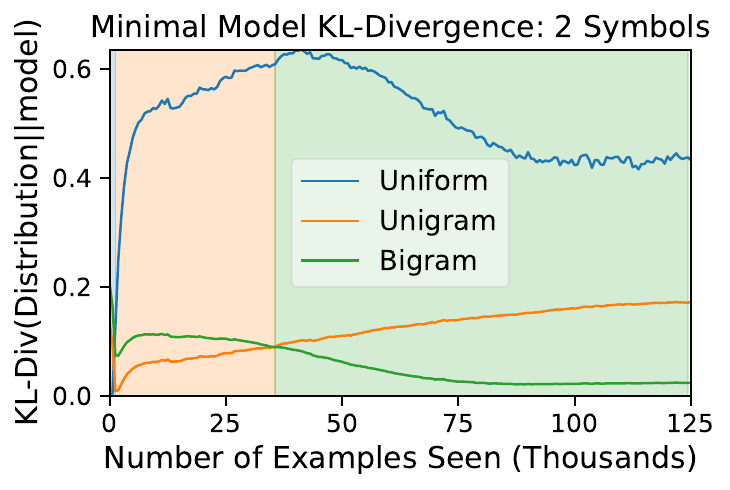}
    \caption{A comparison of the two layer attention only transformer and minimal model for $k=2$ symbols. Note the alternating pattern in the positional encodings at 64,000 examples seen.}
    \label{fig:2symb_pos}
\end{figure*}


\begin{figure*}
    \centering
    \includegraphics[width=\textwidth/2*9/10]{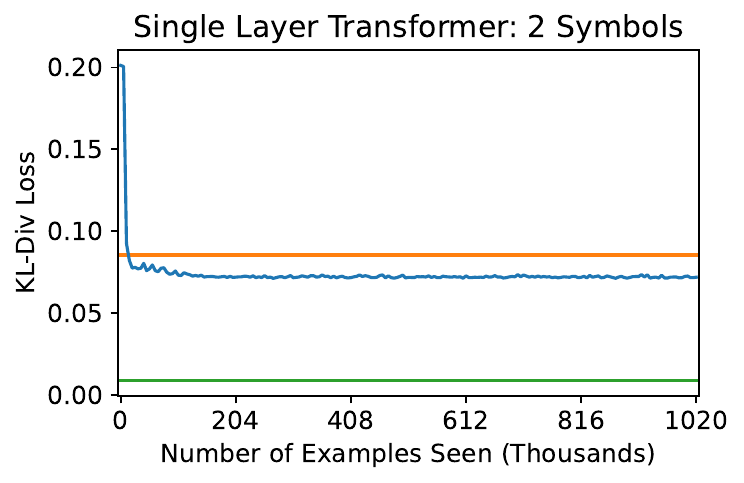}
    \includegraphics[width=\textwidth/2*9/10]{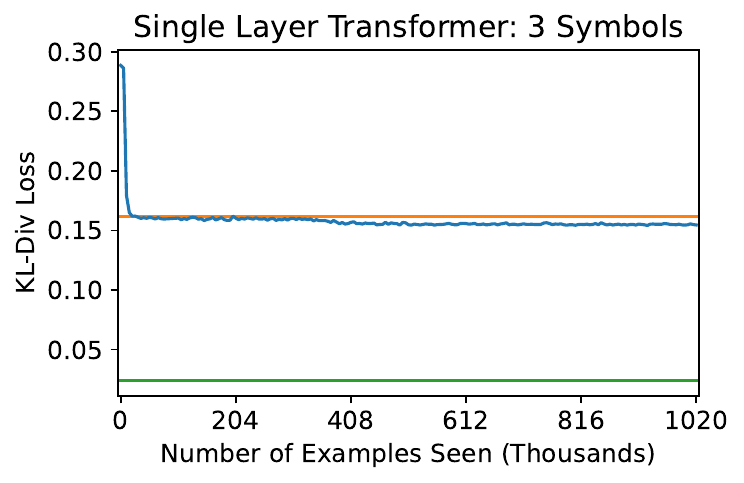}
    \caption{Graphs of test loss showing that a single layer transformer can not achieve good performance on ICL-MC. This result holds for transformers with or without MLPs, and with absolute or relative positional encodings. These graphs show that even trained 8 times longer, there is no notable increase in performance beyond the unigrams strategy (orange line).}
    \label{fig:one_layer}
\end{figure*}

\end{document}